\documentclass[journal]{IEEEtran}
\usepackage{amssymb}
\usepackage{amsmath}
\usepackage{dcolumn}% Align table columns on decimal point
\usepackage{bm}% bold math
\usepackage{color}
\usepackage{graphicx}
\usepackage{subfigure} 
\usepackage{tabularx}
\usepackage{array}
\usepackage{amsthm}
\usepackage{cite}
\usepackage[ruled,linesnumbered]{algorithm2e}
\usepackage{tabu}                     % 表格插入
\usepackage{multirow}                 % 一般用以设计表格，将所行合并
\usepackage{multicol}                 % 合并多列x
\usepackage{multirow}                % 合并多行
\usepackage{float}                    % 图片浮动
\usepackage{makecell}                 % 三线表-竖线
\usepackage{booktabs}
\usepackage{threeparttable} 
\usepackage{nonfloat} 
\usepackage{tablefootnote}
\usepackage{footnotehyper}
\usepackage{xcolor}
\usepackage{diagbox}
\usepackage{array} 
\usepackage{tikz} 
\usepackage{pifont}

\usepackage{etoolbox}
\makeatletter
\patchcmd{\@makecaption}
  {\scshape}
  {}
  {}
  {}
\makeatother

\SetKwInput{KwData}{Input}
\SetKwInput{KwResult}{Output}

\usepackage[colorlinks,
            linkcolor=red,
            anchorcolor=green,
            citecolor=green
            ]{hyperref}

\newtheorem{theorem}{Theorem}
\newtheorem{corollary}{Corollary}
\newtheorem{lemma}{Lemma}
\newtheorem{remark}{Remark}
\newtheorem{definition}{Definition}
\newtheorem{assumption}{Assumption}
\newtheorem{proposition}{Proposition}
\newtheorem{example}{Example}
\newtheorem{problem}{Problem}

\renewcommand{\t}{^{\mbox{\tiny\sf T}}}
\newcommand{\bremark}{\begin{remark}
\begin{rm}}
\newcommand{\eremark}{ \end{rm}\hfill \rule{1mm}{2mm}
\end{remark} }
\newcommand{\btheorem}{\begin{theorem} \begin{it}}
\newcommand{\etheorem}{\end{it} \hfill \rule{1mm}{2mm}
\end{theorem} }
\newcommand{\blemma}{\begin{lemma} \begin{it} }
\newcommand{\elemma}{ \end{it} \hfill\rule{1mm}{2mm}
\end{lemma} }
\newcommand{\bcorollary}{\begin{corollary} \begin{it} }
\newcommand{\ecorollary}{ \end{it} \hfill\rule{1mm}{2mm}
\end{corollary} }
\newcommand{\bdefinition}{\begin{definition} }
\newcommand{\edefinition}{ \hfill\rule{1mm}{2mm}
\end{definition} }
\newcommand{\bproposition}{\begin{proposition} }
\newcommand{\eproposition}{\hfill \rule{1mm}{2mm}
\end{proposition} }
\newcommand{\bexample}{\begin{example} \begin{rm}}
\newcommand{\eexample}{ \end{rm} \hfill\rule{1mm}{2mm}
\end{example} }

\newcommand{\bassumption}{\begin{assumption} }
\newcommand{\eassumption}{\hfill \rule{1mm}{2mm}
\end{assumption} }

\newcommand{\balgorithm}{\medskip\begin{algorithm} \rm}
\newcommand{\ealgorithm}{ \hfill \rule{1mm}{2mm}\medskip
\end{algorithm} }

\newcommand{\basm}{\begin{assumption} \begin{rm} }
\newcommand{\easm}{ \end{rm} \hfill\rule{1mm}{2mm}
\end{assumption} }

\begin{document}

\title{Concurrent-Allocation Task Execution for Multi-Robot Path-Crossing-Minimal Navigation in Obstacle Environments}
\author{
\IEEEauthorblockN{
Bin-Bin Hu, Weijia Yao, Yanxin Zhou, Henglai Wei
and
 Chen Lv\IEEEauthorrefmark{1}
}
%\thanks{This work was
%supported in part by A*STAR AME Young Individual Research under
%Grant A2084c0156, and in part by the SUG-NAP Grant of Nanyang
%Technological University. 
%The work of Yao was supported in part by the National Natural Science Foundation of China under Grant 62573182, and in part by the State Key Laboratory of Digital-Intelligent Modeling and Simulation. The work of Wei was supported by the Key R\&D Program of Shandong Province, China, under Grant 2023CXGC010111 (Corresponding author: Chen Lv).
%}
\thanks{
Bin-Bin Hu, Yanxin Zhou,  and Chen Lv are with the School of
Mechanical and Aerospace Engineering, Nanyang Technological University,
Singapore, 637460 (e-mails: \{binbin.hu, yanxin.zhou, lyuchen\}@ntu.edu.sg). 
}
\thanks{Weijia Yao is with the School of Robotics, Hunan University, Hunan 410082, P.R.~China. (e-mail: wjyao@hnu.edu.cn). 
}
\thanks{Henglai Wei is School of Transportation Science and Engineering, Beihang University, Beijing, China, 100191, P.R.~China. (e-mail: weihenglai@gmail.com). }
}

\maketitle

\begin{abstract}
Reducing undesirable path crossings among trajectories of different robots is vital in multi-robot navigation missions, which not only reduces detours and conflict scenarios, but also enhances navigation efficiency and boosts productivity.
Despite recent progress in {\it multi-robot path-crossing-minimal (MPCM) navigation}, the majority of approaches depend on the minimal squared-distance reassignment of suitable desired points to robots directly. However, if {\it obstacles} occupy the passing space, calculating the actual robot-point distances becomes complex or intractable, which may render the {\it MPCM navigation in obstacle environments} inefficient or even infeasible.

In this paper, the concurrent-allocation task execution (CATE) algorithm is presented to address this problem (i.e., {\it MPCM navigation in obstacle environments}). First, the path-crossing-related elements in terms of (i) robot allocation, (ii) desired-point convergence, and (iii) collision and obstacle avoidance are encoded into integer and control barrier function (CBF) constraints. Then, the proposed constraints are used in an online constrained optimization framework, which implicitly yet effectively minimizes the possible {\it path crossings} and trajectory length in obstacle environments by minimizing the desired point allocation cost and slack variables in CBF constraints simultaneously. 
In this way, the {\it MPCM navigation in obstacle environments} can be achieved with {\it flexible spatial orderings.}
Note that the feasibility of solutions and the asymptotic convergence property of the proposed CATE algorithm in {\it obstacle environments} are both guaranteed, and the calculation burden is also reduced by concurrently calculating the optimal allocation and the control input directly without the path planning process. Finally, extensive simulations and experiments are conducted to validate that the CATE algorithm  (i) outperforms the existing state-of-the-art baselines in terms of feasibility and efficiency in {\it obstacle environments},  (ii) is effective in environments with dynamic obstacles and is adaptable for performing various navigation tasks in 2D and 3D, (iii) demonstrates its efficacy and practicality by 2D experiments with a multi-AMR onboard navigation system, and (iv) provides a possible solution to evade deadlocks and pass through a narrow gap.

\end{abstract}


\IEEEpeerreviewmaketitle

\begin{figure}[!htb]
\centering
\includegraphics[width=6cm]{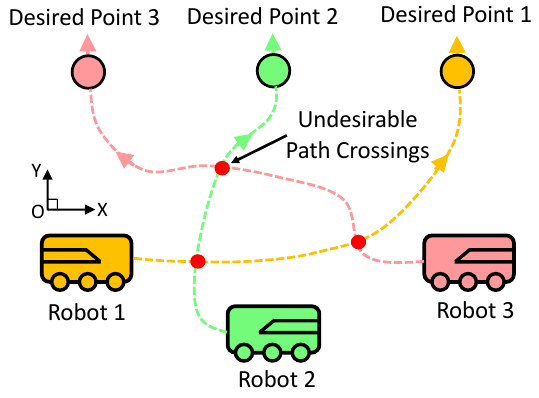}
\caption{Top-down view of the three-robot navigation with fixed-ordering strategies to illustrate the undesirable path crossings.}
\label{Illustration_MPCM}
\end{figure}

\section{Introduction}

In recent years, significant advances have been made in the field of multi-robot navigation research on various application platforms, including unmanned aerial vehicles (UAVs) \cite{yao2021distributed}, unmanned surface vehicles (USVs)~\cite{hu2020multiple}, and autonomous ground vehicles (AGV)~\cite{hang2021cooperative}, leading to multi-robot navigation of higher efficiency, wider ranges, and better resilience. In such a multi-robot navigation problem, robots typically adhere to two subtasks: multi-robot formation maneuvering  (i.e., multiple robots coordinate their motion to form a specific shape and maneuver collectively as a whole), and inter-robot collision avoidance \cite{tanner2005towards,fan2020distributed}. The former involves navigating each robot from its starting position to its desired point to establish the formation pattern and maneuver, whereas the latter focuses on preventing collisions with other robots during the navigation process.

Among these extensive studies, many of them use a fixed-ordering strategy which predetermines a fixed desired point for each robot  \cite{hu2021distributed2, yao2022guiding}. Although such kind of {\it fixed-ordering} strategies are easy to implement, they ignore the undesirable path crossings because the desired points are allocated to fixed robots in advance and robots can only focus on navigating towards the corresponding points. As shown in Fig.~\ref{Illustration_MPCM}, if the desired points are predefined at opposite positions, the robots may get close to each other when they move to their desired points. In this case, since the desired points cannot be changed, the only way for robots to avoid collisions and converge to the desired point is to make a detour or wait, which inevitably causes path crossings, longer trajectories, and more energy consumption~\cite{alonso2017multi,zhou2022swarm}.
%, which may result in detours of longer robot trajectories and more energy consumption~\cite{alonso2017multi,zhou2022swarm}. 
%Therefore, an interesting problem is how to reduce the path crossings among the trajectories of different robots. This problem is referred to as {\it the multi-robot path-crossing-minimal (MPCM) navigation} problem, 
To reduce such undesirable path crossings in Fig.~\ref{Illustration_MPCM}, an interesting problem of {\it multi-robot path-crossing-minimal (MPCM) navigation} is raised, which not only involves the traditional navigation between the robots and fixed desired points but also includes the task allocation to reassign desired points to suitable robots, thereby minimizing path crossings and reducing trajectory lengths and energy consumption. For instance, in warehouses, factories, or search operations where multiple robots need to move simultaneously without obstructing each other’s paths, the superiority of {\it MPCM navigation} still shine, despite that the multi-robot navigation problems with collision avoidance
have been extensively explored in multi-agent path finding (MAPF) and multi-agent task assignment (MATA) \cite{li2021lifelong,chen2021integrated}. Precisely, the MAPF primarily focuses on planning collision-free paths from the start locations to the goal locations while minimizing travel time~\cite{li2021lifelong}, which still adheres to fixed-ordering strategies. While the upper-level MATA concentrates more on efficient task allocation and deadlock resolution~\cite{chen2021integrated}, it pays little attention to the lower-level path-finding problems. In contrast to MAPF and MATA, the {\it MPCM navigation} features the common interest of improving efficiency, but approaches it differently by reducing path crossings to simultaneously consider task allocation and path-finding. In this way, the {\it MPCM navigation} not only achieves energy savings by optimizing robot movements but also fosters reliability and smooth operations by reducing conflicts and deadlock.

%the {\it MPCM navigation} can optimize robot movements to ensure efficient navigation without collisions, which leads to significant energy savings and overall productivity improvements. (ii) In motion planning and traffic management where multiple robots navigate in shared spaces concurrently, conflicts or deadlock situations may often arise. The {\it MPCM navigation} can reduce the likelihood of conflicts and deadlock scenarios, which not only enhances reliability but also fosters smoother coordination.

\subsection{Related works}
\label{sec_related_work}
Existing approaches to achieving {\it MPCM navigation} can be generally categorized into two classes: 1) goal
assignment and motion planning; 2) goal-assignment-free methods~\cite{sun2023mean}.

1) For goal assignment and motion planning, the key idea is to assign suitable desired points to robots because the path crossings are arduous to calculate directly, which thus can only be implicitly reduced in the {\it MPCM navigation} process. In this pursuit, a real-time controller consisting of goal assignment, path planning, and local reciprocal collision avoidance was proposed in \cite{alonso2012image} for shape forming. A coordinated algorithm was designed in \cite{lindsey2012construction} for teams of quadrotor helicopters to construct special cubic structures. However,  these two works~\cite{alonso2012image,lindsey2012construction} feature the decoupled strategy of first assignment and then trajectory planning, which may result in high computational complexity. Later, a decentralized algorithm was designed in \cite{turpin2014capt} for the concurrent assignment and the planning of trajectories (CAPT) for multiple robots. For the elimination of possible conflicts, a conflict-graph-based method was proposed in \cite{sabattini2017optimized}. A clique-based distributed assignment approach was proposed in \cite{sakurama2020multi} to achieve distributed coordination assignment. A recent work~\cite{quan2023robust} proposed a swarm reorganization method to adaptively adjust the formation and task assignments in dense environments. 

In addition to the established works in unstructured space~\cite{alonso2012image,lindsey2012construction,turpin2014capt,sabattini2017optimized,sakurama2020multi,quan2023robust}, significant endeavors have also been devoted to the {\it MPCM navigation} in some spaces with special structures, such as the grid maps. For instance, a fully distributed algorithm was proposed in \cite{wang2020shape} to actively refine the goal assignment to achieve collision-free formation shape. For the integrated task assignment, path planning, and coordination problem, an integrated optimization based on a conflict graph algorithm was designed in \cite{liu2021integrated}. Unfortunately, the assignment in previous {\it MPCM navigation}~\cite{alonso2012image,lindsey2012construction,turpin2014capt,sabattini2017optimized,sakurama2020multi, wang2020shape,liu2021integrated,quan2023robust} heavily relies on the distances between the desired points and robots in open environments, which cannot be adapted to the complicated scenario where obstacles occupy the passing space. Since the calculation of the actual robot-point distances becomes complex or even intractable in such a situation, the solution to the {\it MPCM navigation in obstacle environments} is inefficient or even infeasible.

2) For goal-assignment-free methods in achieving {\it MPCM navigation}, the multi-robot shape is generally achieved with flexible spatial ordering sequences. For instance, a collective algorithm was developed in \cite{rubenstein2014programmable} to govern a swarm with a thousand robots to cooperatively form navigation shapes. For general navigation shapes, a reaction-diffusion network to mimic natural behaviors was proposed in \cite{slavkov2018morphogenesis,reyes2014flocking}. For some special applications of target convoying, considerable efforts have been devoted to constant-velocity and varying-velocity target convoying, respectively~\cite{hu2021bearing,hu2023cooperative}.  For other goal-assignment-free navigation missions in dynamic environments, 
a deep reinforcement learning (DRL) framework was proposed in \cite{han2020cooperative}. A distributed guiding-vector-field algorithm was developed in \cite{huspontaneous2023} to form a {\it flexible-ordering} platoon while maneuvering along a predefined path. Later, it was extended to a {\it flexible-ordering} surface navigation \cite{hu2024coordinated}. A coordinated controller based on mean-shift exploration was designed in \cite{sun2023mean} to empower robot swarms to assemble highly complex shapes with strong adaptability.
Despite that these goal-assignment-free works \cite{rubenstein2014programmable, slavkov2018morphogenesis,reyes2014flocking,hu2021bearing,hu2023cooperative,han2020cooperative,huspontaneous2023,sun2023mean,hu2024coordinated} have the flexibility to some extent, they have not leveraged the {\it flexible-ordering} mechanism to explicitly minimize path crossings due to the arduous calculation of path crossings. Moreover, as the {\it flexible-ordering} coordination often features arbitrary ordering sequences, the same-shape {\it MPCM navigation} can be achieved with various different ordering sequences as well, but it remains unclear how to find the most appropriate one. Once the obstacles are added to the environment, the {\it MPCM navigation} problem will become more intricate, and some goal-assignment-free methods may fail to minimize path crossings.

Therefore, considering the aforementioned shortcomings of potential infeasibility and low efficiency in goal-assignment approaches~\cite{alonso2012image,lindsey2012construction,turpin2014capt,sabattini2017optimized,sakurama2020multi, wang2020shape,liu2021integrated,quan2023robust} and unclear path-crossing-minimal mechanisms in goal-assignment-free approaches~\cite{rubenstein2014programmable, slavkov2018morphogenesis,reyes2014flocking,hu2021bearing,hu2023cooperative,han2020cooperative,huspontaneous2023,sun2023mean,hu2024coordinated}, the challenging problem of {\it MPCM navigation in obstacle environments} still requires further investigation.

\subsection{Contribution} 
%Recently, a minimum-energy task execution was developed in \cite{notomista2019constraint,notomista2019optimal} to enable the long-term execution of the MRN, which later was extended to the heterogeneous multi-robot teams for executing multiple different tasks~\cite{notomista2021resilient}. 

Inspired by the minimum-energy task execution with feasibility guarantees in \cite{notomista2019constraint} and task allocation in \cite{notomista2021resilient}, 
we propose the CATE algorithm to achieve efficient {\it MPCM navigation in obstacle environments}. First, the path-crossing-related elements in terms of (i) robot allocation, (ii) desired-point convergence, and (iii) collision and obstacle avoidance are encoded into integer and control barrier function (CBF) constraints. Then, the proposed constraints are used in an online optimization framework, which implicitly yet effectively minimizes the possible {\it path crossings} and trajectory length in obstacle environments by minimizing the desired point allocation and slack variables in CBF constraints simultaneously.  The main contributions are summarized in four-fold.
\begin{enumerate}
  
 \item We formulate the {\it MPCM navigation in obstacle environments} to be an online constrained optimization problem, and design the CATE algorithm which governs a team of robots to achieve efficient {\it MPCM navigation} with an arbitrary spatial ordering.

\item Unlike previous {\it MPCM navigation} works \cite{alonso2012image,lindsey2012construction,turpin2014capt,sabattini2017optimized,sakurama2020multi, wang2020shape,liu2021integrated,quan2023robust,rubenstein2014programmable, slavkov2018morphogenesis,reyes2014flocking,hu2021bearing,hu2023cooperative,han2020cooperative,huspontaneous2023,sun2023mean,hu2024coordinated}, which only focus on open environments and suffer from low navigation efficiency due to additional obstacles, we efficiently achieve the {\it MPCM navigation in obstacle environments} with fewer path crossings and shorten trajectory lengths by minimizing the desired point allocation cost and slack variables in the CBF constraints simultaneously.

\item Compared with the potential infeasibility and low computational efficiency caused by the minimization of possibly intractable robot-point distances in previous goal-assignment and motion-planning approaches \cite{alonso2012image,lindsey2012construction,turpin2014capt,sabattini2017optimized,sakurama2020multi, wang2020shape,liu2021integrated,quan2023robust}, we can guarantee the feasibility of solutions to the problem of {\it MPCM navigation in obstacle environments} by incorporating the minimum-energy task execution framework, and require less calculation burden by concurrently calculating the optimal allocation and the control input directly without path (re)planning.

 \item We conduct extensive simulations and experiments to validate the CATE algorithm that (i) outperforms the existing state-of-the-art baselines in terms of feasibility and efficiency in obstacle environments,  (ii) shows its effectiveness to deal with dynamic obstacles and adaptability for performing various navigation tasks in 2D and 3D, (iii) demonstrates its efficacy and practicality by 2D experiments with a multi-AMR onboard navigation system, and (iv) provides a possible solution to evade deadlocks and pass through a narrow gap.

\end{enumerate}

The remainder of this paper is organized below. Section~\ref{sec_pre} introduces the preliminaries. Then, Section~\ref{sec_problem} formulates the {\it MPCM navigation} problem. Section~\ref{sec_CATE} proposes and illustrates the CATE algorithm. Section~\ref{sec_main} provides the detailed features and convergence analysis of the algorithm. Section~\ref{sec_verification} presents extensive simulation examples and AMR experiments. Finally, Section~\ref{sec_conclusion} concludes the paper.

\textbf{Notations}: The notations $\mathbb{R},\mathbb{R}^+$ denote the real numbers and positive real numbers, respectively. The notation $\mathbb{R}^n$ represents the $n$-dimensional Euclidean space. The notations $\mathbb{Z}$ and $\mathbb{Z}_i^j$ represent the integer number and the integer set $\{m\in \mathbb{Z}~|~i\leq m\leq j\}$, respectively. The Kronecker product is denoted by~$\otimes$. The $n$-dimensional identity matrix is represented by~$I_n$. The $N$-dimensional column vector consisting of all 1 is denoted by $\mathbf{1}_N$.

\section{Preliminaries}
\label{sec_pre}
In this section, we introduce necessary building blocks and definitions for {\it MPCM navigation in obstacle environments}.

\subsection{Multi-Robot Systems}
The first block is a team of $N$ robots with the index set $\mathcal V:=\{1,2, \dots, N\}$, where each robot moves with the single-integrator dynamic \cite{notomista2019constraint}, 
\begin{align}
\label{robot_dynamic}
\dot{\bold{x}}_i=\bold{u}_i, \|\bold{u}_i\|\leq u_{\max}, i\in\mathcal V,
\end{align}
where $ \bold{x}_i:=[x_{i,1}, \dots, x_{i,n}]\in\mathbb{R}^n$ and $\bold{u}_i=[u_{i,1}, \dots, u_{i,n}]\in\mathbb{R}^n$ are the position and input of robot $i$, respectively, and $u_{\max}\in\mathbb{R}^{+}$ is a positive constant. 
Here, the dimension $n\in\mathbb{R}^+$ of $\bold{x}_i, \bold{u}_i$ in \eqref{robot_dynamic} represents the dimension of robot operation in the Euclidean space, which is different from the number of robots $N$. For instance, $n=2$ if robots, such as ground vehicles, and unmanned surface vessels (USVs) \cite{hang2021cooperative,hu2021distributed1} operate in 2D planes, and $n=3$ if robots, such as unmanned aerial vehicles (UAVs) \cite{yao2022guiding} operate in 3D. 
It is worth mentioning that the velocity input $\bold {u}_i$ can be treated as a high-level command when encountering robots of higher-order complicated dynamics. 
%\begin{remark}
%The reason we consider the simple single-integrator model in \eqref{robot_dynamic} is that the velocity input $u_i$ is treated as a high-level command to robots, while the specific low-level robot dynamics can be more complicated than the simple model \eqref{robot_dynamic}, such as ground vehicles, fixed-wing aircrafts, and unmanned surface vessels \cite{hang2021cooperative,yao2022guiding,hu2021distributed1}. In Section~\ref{sec_verification}, we use the non-holonomic unicycle model to verify our proposed algorithm's effectiveness in  both simulations and experiments.
%\end{remark}
Let $\mathcal N_i(t)$ be the sensing neighbor set of robot $i$ as
\begin{align}
\label{neighbor_set}
\mathcal N_i(t)=\{j\in\mathcal V~\big|~\| \bold{x}_i(t)-\bold{x}_j(t) \|\leq R\},
\end{align}
where $t$ is defined to be the time throughout the paper, and $R\in(r,\infty)$ is a sensing radius with $r\in\mathbb{R}^+$ being a specified safe radius. Since $x_i(t)$ is time-dependent, one has that $\mathcal N_i(t)$ in \eqref{neighbor_set} is time-dependent as well. The geographical condition of $R>r$ is commonly used in previous multi-robot works  (see \cite{chen2019cooperative,huspontaneous2023,hu2024fsoc}), which not only ensures that the robot detects other neighboring robots before possible collision, but also activates the neighboring collision-avoidance CBF constraints to avoid collision \cite{wang2017safety,hu2024ordering}. Moreover, with such a condition, even if another robot $j\neq i$ is getting too close to robot $i$ (i.e., collide), it follows from $\mathcal N_i(t)$ that $\|\bold{x}_i(t)-\bold{x}_j(t)\|\leq r<R$, which implies that robot $j$ is still a neighbor of robot $i$.

%According to the definition of $\mathcal N_i$ in~\eqref{neighbor_set}, one has that $\mathcal N_i$ is time-varying. which will be utilized to establish neighboring collision-avoidance CBF constraints in~\eqref{task_2_CBF_constraint} later.

\subsection{Desired Points}
\label{second_block}
The second block is $N$ desired points of the navigation formation, where each point $\bold{x}_i^d\in\mathbb{R}^n$ satisfies
\begin{align}
\label{desired_velocity}
\dot{\bold{x}}_i^d=\bold{v}_d(t),  i\in\mathcal V,
\end{align}
with $\bold{v}_d(t)\in\mathbb{R}^n$ being a predefined  constant or non-constant desired velocity. 
Here, the superscript and subscript $d$ in $\bold{x}_i^d, \bold{v}_d(t)$ is defined to represent the positions and velocities of desired points, which is used to distinguish from those of robots. $\bold{v}_d(t)$ is assumed to be continuously differentiable. Since $\bold{v}_d(t)$ in~\eqref{desired_velocity} is the same for all desired points, the desired navigation maintains a rigid formation. Note that 
$\|\bold{v}_d(t)\|< u_{\max}$ with the limit $u_{\max}$ given in \eqref{robot_dynamic}, which implies that all robots can catch up and form the desired formation. Moreover, the desired formation can maneuver with a varying or constant velocity by setting $\dot{\bold{v}}_d(t)\neq0$ or $\dot{\bold{v}}_d(t)=0$.

\begin{figure}[!htb]
\centering
\includegraphics[width=6cm]{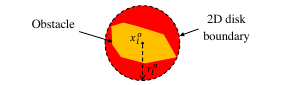}
\caption{Simple interpretation of a 2D disk covering complex irregular obstacles.}
\label{Obstacle_boundary}
\end{figure}

\begin{figure}[!htb]
\centering
\includegraphics[width=7.2cm]{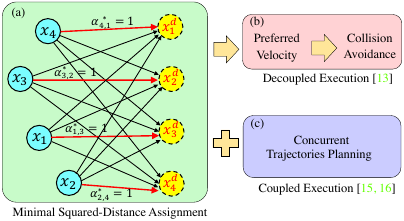}
\caption{Two existing methods for achieving navigation with non-intersecting paths. (a, b) Decoupled execution of the minimal squared-distance assignment, preferred-velocity calculation, and inter-robot collision avoidance in order \cite{alonso2012image}.  (a, c) Concurrent execution of the minimal squared-distance assignment and collision-free trajectory planning~\cite{turpin2014capt,sabattini2017optimized}. The blue and yellow circles denote the initial positions of robots and the final positions of desired points, respectively. The black and red arrows denote the non-selected and optimal robot-point pairs, respectively. $\alpha_{i,k}^{\ast}=1$ the correspondence between robot~$i$ and desired point $k$ in~\eqref{eq_con_a_poi_a_rob}.  }
\label{previous_PCM}
\end{figure}

%\begin{remark}
%The desired point $x_i^d, i\in\mathcal V,$ in~\eqref{task_point_i} is generally predefined and fixed for each robot in a classic navigation task (i.e., $x_i^d$ for robot $i$), which thus is easily applied in practice.
%Moreover, if at least one robot knows its desired point $x_i^d$, other robots can calculate their desired points by the desired inter-point relative positions via a connected communication network \cite{hu2021distributed1,hu2021distributed2}.
%\end{remark}

\subsection{Allocation Matrix}
\label{subsec_allo}
With the robots in \eqref{robot_dynamic} and the desired points in \eqref{desired_velocity}, the third block is to define an allocation matrix $\bm\alpha:=[\bm\alpha_1, \cdots, \bm\alpha_N]\t\in\mathbb{R}^{N\times N}$ with the $i$-th column vector $\bm\alpha_i:=[\alpha_{i,1}, \cdots, \alpha_{i,k}, \cdots, \alpha_{i,N}]\t\in\mathbb{R}^{N}$~\cite{notomista2021resilient}. Here, each entry $\alpha_{i,k}, i\in\mathcal V, k\in\mathbb{Z}_1^N$ satisfies
\begin{equation}
\label{robot_i_entry_priority}
\alpha_{i,k}=\left\{
\begin{aligned}
1, & &\mathrm{if~desired~point}~{\it k}~\mathrm{ is~allocated~to~robot~{\it i},}\\
0, & &\mathrm{otherwise.}
\end{aligned}
\right.
\end{equation}
To achieve the required pattern, we define $\bold{1}_N=[1, \cdots, 1]\t\in\mathbb{R}^N$ and require $\bm\alpha$ to fulfill the following two conditions:
\begin{enumerate}
\item \label{con_a_rob_a_poi}  Each robot tracks only one desired point, i.e., 
\begin{align}
\label{eq_con_a_rob_a_poi}
\bm\alpha \bold{1}_N=\bold{1}_N.
\end{align}

\item \label{con_a_poi_a_rob} Each desired point is allocated to only one robot, i.e.,
\begin{align}
\label{eq_con_a_poi_a_rob}
\bold{1}_N\t \bm\alpha =\bold{1}_N\t.
\end{align}
\end{enumerate}
The conditions in \eqref{eq_con_a_rob_a_poi} and \eqref{eq_con_a_poi_a_rob} stipulate a basic constraint of the one-to-one correspondence between robots and desired points. There still exist $N!$ candidates to be explored to find the optimal $\bm\alpha^{\ast}$, which thus offers flexibility in terms of ordering sequences in {\it MPCM navigation}.

\subsection{Obstacles}
The fourth block is a set of $M$ finite-size obstacles $\mathcal O_{all}=\{\mathcal O_l\subseteq \mathbb{R}^n~|~l\in\mathcal M\}$ with the index set $\mathcal M=\{1,2, \dots, M\}$. 

\begin{figure}[!htb]
\centering
\includegraphics[width=5.5cm]{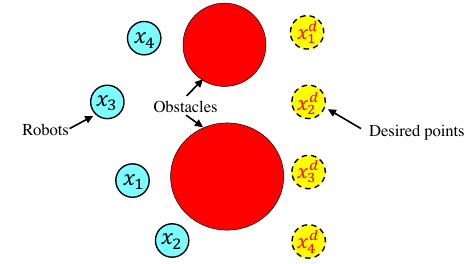}
\caption{Failure of the minimal squared-distance assignment between robots and desired points in Fig.~\ref{previous_PCM}  when encountering obstacles occupying the space.}
\label{failed_assigment}
\end{figure}

To prevent the complicated calculation of specific-shape obstacles, we choose $\bold{x}_l^o \in \mathbb{R}^n$ and $r_l^o \in\mathbb{R}^+$ such that the obstacle $\mathcal O_l$ is contained in the 2D disk/ 3D ball centered at $\bold{x}_l^o$ with radius $r_l^o$:
\begin{align}
\label{obstacle_boundary}
\mathcal O_l\subseteq \{\bm{\sigma}_l\in\mathbb{R}^n~|~ \|\bm{\sigma}_l- \bold {x}_l^o\|\leq r_l^o\}, l\in\mathcal{M},
\end{align} 
where $\bm\sigma_l:=[\sigma_{l,1}, \sigma_{l,2}, \cdots, \sigma_{l,n}]\t$ are the coordinates. Then, avoiding collision with $\mathcal O_l$ can be simplified to avoiding    
collisions with the prescribed disks/balls, as shown in Fig.~\ref{Obstacle_boundary}.

\begin{remark}
The obstacle element is necessary, yet challenging in the {\it MPCM navigation} problem. As shown in Fig.~\ref{previous_PCM}, most of the previous works~\cite{alonso2012image,turpin2014capt,sabattini2017optimized} have investigated the scenario of generating non-intersecting paths, which focused on finding the optimal allocation matrix $\bm\alpha^{\ast}$ in \eqref{robot_i_entry_priority} by minimizing the sum of squared distance between arbitrary robots $\bold{x}_i$ and desired goals $\bold{x}_k^d$ subject to two conditions~\eqref{eq_con_a_rob_a_poi}, \eqref{eq_con_a_poi_a_rob},
i.e., $\bm\alpha^{\ast}:=\mathop{\arg\min}\limits_{\bm\alpha}\sum_{i=1}^{N}\sum_{k=1}^N\alpha_{i,k}\| \bold{x}_i-\bold{x}_k^d\|^2,~\mathrm{s.t.}~\bm\alpha \bold{1}_N=\bold{1}_N, \bold{1}_N\t \bm\alpha =\bold{1}_N\t,$
%\begin{align}
%\bm\alpha^{\ast}:=&\mathop{\arg\min}\limits_{\bm\alpha}\sum_{i=1}^{N}\sum_{k=1}^N\alpha_{i,k}\|{\blue \bold{x}_i-\bold{x}_k^d}\|^2\nonumber\\
%&\mathrm{s.t.}~\bm\alpha 1_N=1_N, 1_N\t \bm\alpha =1_N\t,
%\end{align} 
and then use $\bm\alpha^{\ast}$ to (re)-plan collision-free trajectories in a decoupled \cite{alonso2012image} or coupled manner \cite{turpin2014capt,sabattini2017optimized}. However, when obstacles occlude the space between robots and desired points in Fig.~\ref{failed_assigment}, such kind of assignment-planning methods \cite{alonso2012image,turpin2014capt,sabattini2017optimized} may fail to work because the actual robot-point distances become sophisticated or intractable to calculate, which may cause deadlock and infeasible solutions.
%Therefore, the {\it MPCM navigation in obstacle environments} still remains an open problem.
\end{remark}

\begin{remark}
The reason why we utilize 2D disks or 3D balls in~\eqref{obstacle_boundary} is that it is simple to design smooth boundaries to enclose the irregular-shape obstacles. %Therefore, these simple boundaries facilitate to encode the obstacle-avoidance CBF constraints in \eqref{task_3_CBF_constraint} later. 
Although such disks/balls may occupy some additional space, it can be seamlessly incorporated into the minimum-energy task execution in Definition~\ref{definition_minimum_energy} to guarantee the feasibility of {\it MPCM navigation} in obstacle environments. 
Moreover, the use of disks/balls can also be replaced by ellipses/ellipsoids to save spaces, such as $\mathcal O_l\subseteq \{[\sigma_{l,1}, \sigma_{l,2}]\in\mathbb{R}^2~|~{(\sigma_{l,1}-x_{l,1}^o)^2}/{(r_{l,1}^o)^2} + {(\sigma_{l,2}-x_{l,2}^o)^2}/{(r_{l,2}^o)^2} \leq1\}, l\in\mathcal{M},$ and the 3D ellipsoid: $\mathcal O_l\subseteq \{[\sigma_{l,1}, \sigma_{l,2}, \sigma_{l,3}]\in\mathbb{R}^3~|~{(\sigma_{l,1}-x_{l,1}^o)^2}/{(r_{l,1}^o)^2} + {(\sigma_{l,2}-x_{l,2}^o)^2}/{(r_{l,2}^o)^2} +  {(\sigma_{l,2}-x_{l,3}^o)^2}/{(r_{l,3}^o)^2}   \leq1\}, l\in\mathcal{M}$ with $x_{l,1}\in\mathbb{R}, x_{l,2}\in\mathbb{R}, x_{l,3}\in\mathbb{R}$ and  $r_{l,1}\in\mathbb{R}^+, r_{l,2}\in\mathbb{R}^+, r_{l,3}\in\mathbb{R}^+$ being their centers and lengths of the semi-axes, respectively. For simplicity, we use disks and balls such that we can focus on the core idea of the obstacle avoidance tasks in {\it MPCM navigation in obstacle environments}. 
\end{remark}

\subsection{Minimum-Energy Task Execution}
The last block is the minimum-energy task execution, which will be utilized to implicitly prevent calculating the sophisticated robot-point distances.
Before proceeding, suppose the desired task $\mathcal T\subseteq\mathbb{R}^n$ is characterized by a target set~\cite{notomista2019optimal}
\begin{align}
\label{task_set}
\mathcal T=\{\bm\sigma\in\mathbb{R}^n~\big|~\phi(\bm\sigma)\leq 0\},
\end{align}
where $\phi(\bm\sigma): \mathbb{R}^n\rightarrow\mathbb{R}$ is a twice continuously differentiable function with the coordinates $ \bm\sigma:=[\sigma_{1}, \sigma_{2}, \cdots, \sigma_n]\t$.
%The expression of $\mathcal T$ in~\eqref{task_set} has several advantages. Firstly, one can use
Firstly, one can use $\phi(\bm\sigma)$ to describe the target sets conveniently for different tasks. 
Secondly, one can utilize $\phi(\bold{p}_0)$ to approximate the sophisticated projected distance $\mathrm{dist}(\bold{p}_0, \mathcal T):=\mathrm{inf}\{\|\bold{p}-\bold{p}_0\|~\big|~p\in\mathcal T\}$ between a point $\bold{p}_0\in\mathbb{R}^n$ and the target set $\mathcal T$. In addition,  one can use $\phi(\bm\sigma)$ as a control barrier function (CBF)~\cite{ames2016control}, which is an essential ingredient of the minimum-energy task execution in~Definition~\ref{definition_minimum_energy}.

%Analogous to the desired task in~\eqref{task_set}, if $\mathcal T:=\{\sigma\in\mathbb{R}^n~\big|~\phi(\sigma)\leq 0\}$ by a continuously
%differentiable cost function $\phi(\sigma)$, it follows from~\eqref{robot_dynamic} that the typical execution of $\mathcal T$ for all robots is to minimize $\phi(x_i)$ \cite{quan2022distributed}, i.e.,
%\begin{align}
%\label{tradi_minimization}
% \min_{u_i}~&\phi^2(x_i)\nonumber\\
%\mathrm{s.t.}~&\dot{x}_i=u_i, i\in\mathbb{Z}_1^N.
%\end{align}
%With $u_i$ calculated for robot $i$ in~\eqref{tradi_minimization}, we can drive the robot to converge to the target set $\mathcal T$, i.e., $\{\phi(x_i)\le0\}$. However, such typical minimization \eqref{tradi_minimization} only tries the best to minimize $\phi^2(x_i)$ rather than the control inputs $u_i$ directly, which fails to be combined with the allocation matrix $\bm\alpha$ in Section~\ref{sec_problem_formulation}.

\begin{definition}
\label{definition_minimum_energy}
(Minimum-energy task execution) \cite{notomista2019constraint}. For a desired task $\mathcal T$ given in \eqref{task_set}, a team of robots $\mathcal V$ governed by~\eqref{robot_dynamic} achieves
the minimum-energy task execution if the following constrained optimization problem is solved, i.e., 
\begin{subequations}
\label{constraint_minimization}
\begin{align}
& \min_{\bold{u}_i, \bm\delta_i}~\|\bold{u}_i\|^2+|\bm\delta_i|^2\\
\mathrm{s.t.}~& \frac{\partial h(\bold{x}_i)}{\partial \bold{x}_i\t}{\bold{u}_i}\geq-\gamma\big(h(\bold{x}_i)\big)-\bm\delta_i, i\in\mathbb{Z}_1^N,
\end{align}
\end{subequations}
where $h(\bold{x}_i)=-\phi(\bold{x}_i)$ satisfying the constraint in (\ref{constraint_minimization}b) is a CBF \cite{ames2016control}, $\bm\delta_i\in\mathbb{R}^+$ is a slack variable to measure the extent of task violation, and $\gamma(\cdot): (-b, a)\rightarrow\mathbb{R}$ is an extended class $\mathcal K$ function with $a, b\in\mathbb{R}^+$, which is strictly increasing and $\gamma(0)=0$~\cite{ames2016control}. Here, $\gamma\big(h(\bold{x}_i)\big)$ can regulate how fast the robots approach the boundary of the target set in~\eqref{task_set}.
\end{definition}

In Definition~\ref{definition_minimum_energy}, it is straightforward that the optimization in \eqref{constraint_minimization} always contains a feasible solution $\{\bold{u}_i^{\ast}=0, \bm\delta_i \mbox{ is sufficiently large}\}$, which guarantees the feasibility of {\it MPCM navigation in obstacle environments}. Moreover, it has been shown in \cite{notomista2019constraint,xu2015robustness} that the asymptotic convergence of the minimum-energy task execution \eqref{constraint_minimization} is achieved, i.e.,  
\begin{align}
\label{convergence_property}
\bold{x}_i(0)\notin \mathcal T \Rightarrow  \bold{x}_i(t)\in \mathcal T, t\rightarrow\infty,
\end{align}
if the optimal input $\bold{u}_i^{\ast}(\bold{x}_i)$ in~\eqref{constraint_minimization} are locally Lipschitz  continuous with respect to its arguments. 
By encoding $\phi(\bold{x}_i)$ into the CBF constraint in~\eqref{constraint_minimization} rather than minimizing $\phi(\bold{x}_i)$ in the cost function \cite{emam2021data}, Definition~\ref{definition_minimum_energy} ensures the forward invariance of the task set $\mathcal T$ \cite{ames2016control}, i.e., 
\begin{align}
\label{invariant_property}
\bold{x}_i(t_1)\in \mathcal T \Rightarrow  \bold{x}_i(t)\in \mathcal T, \forall t>t_1,
\end{align}
with a constant time $t_1\in\mathbb{R}^+$ and locally Lipschitz continuous input $\bold{u}_i^{\ast}(\bold{x}_i)$ as well. %It may have the robustness property in changing environmental conditions~\cite{notomista2021resilient}. 

\section{Problem formulation}
\label{sec_problem}
Using the necessary blocks in~Section~\ref{sec_pre}, we can introduce the problem of  {\it MPCM navigation in obstacle environments}.

\begin{definition}
\label{definition_MPCM}
({\it MPCM navigation in obstacle environments}) A~team of robots $\mathcal V$ governed by~\eqref{robot_dynamic} collectively form the path-crossing-minimal navigation to the desired points governed by \eqref{desired_velocity} in obstacle environments with~\eqref{obstacle_boundary}, if the following properties are fulfilled, 
\begin{itemize}
\item  {\bf P1 (Crossing Minimum):}  \label{P_1}All robots $\mathcal V$ minimize the path crossings during the navigation process in obstacle environments, i.e.,  $\min |\mathcal C|$, where the set $\mathcal C$ containing all path crossings satisfies $\mathcal C=\cup_{i\in\mathcal V}\{ \bold{x}_i(t_1)\in\mathbb{R}^n~\big|~\exists j\neq i\in\mathcal V, t_1, t_2\in\mathbb{R}^+, \mathrm{such~that~} \bold{x}_i(t_1)=\bold{x}_j(t_2)\} $
%\begin{align}
%\label{set_path_crossing}
%\mathcal C=&\cup_{i\in\mathcal V}\{{\blue \bold{x}_i}(t_1)\in\mathbb{R}^n~\big|~\exists j\neq i\in\mathcal V, t_1, t_2\in\mathbb{R}^+, \nonumber\\
%& \mathrm{such~that~}{\blue \bold{x}_i(t_1)=\bold{x}_j(t_2)}\} 
%\end{align}
with~$\bold{x}_i, i \in \mathcal V$~being~the~trajectory~of~the~$i$-th~robot and $t_1, t_2$ the different time instances.

\item\label{P_2} {\bf P2 (Convergence and maneuvering):} All robots converge to their desired points and maneuver with the desired velocity $\bold{v}_d$ in \eqref{desired_velocity}, i.e., $\lim_{t\rightarrow\infty} \{\bold{x}_i(t)-\bold{x}_k^d(t)\}$ $=0, \lim_{t\rightarrow\infty}\dot{\bold{x}}_i(t)=\bold{v}_d, i\in\mathbb{Z}_1^N,$
%\begin{align*}
%{\blue \lim_{t\rightarrow\infty} \{\bold{x}_i(t)-\bold{x}_k^d(t)\}=0, \lim_{t\rightarrow\infty}\dot{\bold{x}}_i(t)=\bold{v}_d}, i\in\mathbb{Z}_1^N,
%\end{align*} 
where $k$ is the index for which $\alpha_{i,k}=1$ in \eqref{robot_i_entry_priority}.

\item \label{P_3} {\bf P3 (Collision avoidance):} The collision avoidance among neighboring robots is guaranteed all along, i.e., 
$\|\bold{x}_i(t)-\bold{x}_j(t)\|\geq r, \forall t>0, i\in\mathbb{Z}_1^n, j\in\mathcal N_i(t),$
where $r\in\mathbb{R}^+$ is a specified safe radius, and $\mathcal N_i$ is the sensing neighbor set of robot $i$ given in~\eqref{neighbor_set}.

\item {\bf P4 (Obstacle avoidance):} \label{P_4}  The collision between arbitrary robot and obstacles is avoided all along, i.e., 
$\|\bold{x}_i(t)-\bold{x}_l^o(t)\|\geq r+r_l^o, \forall t>0, i\in\mathbb{Z}_1^n, l\in\mathcal M,$
where $\bold{x}_l^o$, and $r_l^o$ are given in \eqref{obstacle_boundary}.

\end{itemize}
\end{definition}

In Definition~\ref{definition_MPCM}, the path crossings $\mathcal C$ in {\bf P1} are spatial, since trajectory crossings happened at different time instances are also counted.  For \textbf{P2-P4}, they constitute the fundamental criteria for multi-robot navigation.

%Moreover, from an efficiency standpoint, this reduction can mitigate detours and minimize travel distances.
%Now, we are ready to formulate the main problem addressed in this paper.

\begin{problem}
\label{label_pro}
Design the allocation matrix and velocity $\{\bm\alpha, \bold{u}_i\}$ for each robot $i$ according to $\bold{x}_i, \bold{x}_i^d, \bold{v}_d, i\in\mathcal V, r_l^o, \bold{x}_l^o, l\in\mathcal M$ in \eqref{robot_dynamic}, \eqref{desired_velocity}, and \eqref{obstacle_boundary} such that the {\it MPCM navigation in obstacle environments} is achieved, i.e., \textbf{P1-P4} in Definition~\ref{definition_MPCM}.
\end{problem}

To address Problem~\ref{label_pro}, several assumptions are required.

\begin{itemize}
\item {\bf A1:} \cite{yao2021singularity} For any $\kappa\in\mathbb{R}^+$, it holds that $\mathrm{inf}\{\phi(\bold{p}_0)~\big|~\mathrm{dist}($ $\bold{p}_0, \mathcal T)\geq\kappa, p_0 \in \mathbb{R}^n \}>0$, where $\mathcal T$ generally refers to the desired target set discussed later.

%$\mathcal T_{i,k}^{[1]}, \mathcal T_{i,j}^{[2]}, \mathcal T_{i,l}^{[3]}$ in Eqs.~\eqref{task_point_i}, \eqref{task_collision_i}, \eqref{task_obstacle_i}.

\item {\bf A2:}  The initial distance between arbitrary two desired points satisfies $\|\bold{x}_i^d(0)-\bold{x}_j^d(0)\|> r, i\neq j\in\mathcal V,$ with $r$ given below \eqref{neighbor_set}.

\item {\bf A3:}  The initial distance between arbitrary two robots satisfies $\|\bold{x}_i(0)-\bold{x}_j(0)\|> r, \forall i\neq j\in\mathcal V.$

\item {\bf A4:}  The initial distance between arbitrary robot $i$ and obstacle $l$ satisfies $\|\bold{x}_i(0)-\bold{x}_l^o(0)\|>r+r_l^o, \forall i\in\mathcal V, l\in\mathcal O$, with $r_l^o$ given in~\eqref{obstacle_boundary}.

\item {\bf A5:}  The sensing radius $R$ in \eqref{neighbor_set} and safe radius $r$ satisfy $R\geq r+\epsilon,$ with an arbitrary positive constant $\epsilon\in\mathbb{R}^+$.

\item {\bf A6:}  For each robot $i$, there exists a continuously differentiable estimator $\widehat{\bold{v}}_i$ for the desired velocity $\bold{v}_d$ in~\eqref{desired_velocity} such that $\lim_{t\rightarrow\infty}\{\widehat{\bold{v}}_i(t)-\bold{v}_d(t)\}=0, \|\widehat{\bold{v}}_i\|\leq u_{\max}$.

\end{itemize}

{\bf A1} rigorously guarantees that $\lim_{t\rightarrow\infty}|\phi(\bold{p}_0(t))|$ $=0\Rightarrow \lim_{t\rightarrow\infty}\mathrm{dist}(\bold{p}_0(t), \mathcal T)=0$. {\bf A2} prevents inter-robot collisions if the desired multi-robot navigation is formed. {\bf A3} and {\bf A4} are necessary conditions for {\bf P3} and {\bf P4}, respectively. 
 {\bf A5} is reasonable and intuitive in practice.
{\bf A6} is a common estimator if a small proportion of robots have access to $\bold{v}_d$~\cite{hong2006tracking,hu2021distributed1}.

\begin{figure}[!htb]
\centering
\includegraphics[width=7.4cm]{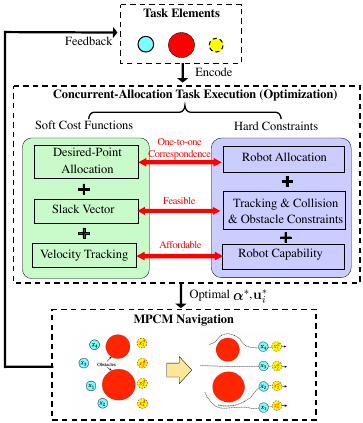}
\caption{Illustration of the core idea of the CATE optimization for the {\it MPCM navigation in obstacle environments}.}
\label{Reactive_PCM}
\end{figure}
\section{CATE Algorithm}
\label{sec_CATE}
In this section, we will introduce the CATE algorithm to address Problem~\ref{label_pro}. A high-level overview of the algorithm is given in Fig.~\ref{Reactive_PCM}, where the core idea is to encode all elements of robots, desired points, and obstacles into subtasks of a concurrent-allocation task execution optimization framework. Therein, the optimization consists of two parts, namely, the soft cost functions and the hard constraints, where the cost functions consist of desired-point allocation, slack vectors, and velocity tracking, and the hard constraints involve robot allocation, tracking \& robot collision \& obstacle avoidance, and robot capability. Firstly, the desired-point allocation and robot allocation are combined to achieve a flexible one-to-one robot-point mapping. Secondly, the slack vectors guarantee the feasibility of the desired-point tracking and robot collision \& obstacle avoidance CBF constraints, which govern robots to converge to the suitable desired points in obstacle environments. Thirdly, the velocity tracking is set to achieve the navigation maneuvering. Finally, the concurrent calculation of the optimal allocation matrix $\bm\alpha^{\ast}$ and inputs $\bold{u}_i^{\ast}$ will be executed to achieve {\it MPCM navigation} in obstacle environments.

%In what follows, we first encode the subtasks of {\it MPCM navigation} into the CBF constraints, and then formulate the CATE algorithm to be a constrained optimization problem. 

\label{sec_main}

\subsection{Encoding CBF Constraints}
%According to Definition~\ref{definition_MPCM},  we introduce three kinds of tasks, namely, the desired-point convergence, inter-robot collision avoidance, and obstacle avoidance. 

Suppose the desired-point convergence task $\mathcal T_{i,k}^{[1]}$ between robot~$i, i\in\mathcal V,$ and the desired point $\bold{x}_k^d, k\in\mathcal V,$ in~\eqref{desired_velocity} is 
\begin{align}
\label{task_point_i}
\mathcal T_{i,k}^{[1]}=&\big\{\bm\sigma_i:=[\sigma_{i,1}, \cdots, \sigma_{i,n}]\t\in\mathbb{R}^n\nonumber\\
&~\big|~\phi_{i,k}^{[1]}(\bm\sigma_i):=\| \bm\sigma_i-\bold{x}_k^{d}\|\leq 0\big\}, i\in\mathcal V,
\end{align}
where $\bm\sigma_i$ are the coordinates, $\phi_{i,k}^{[1]}$ represents a distance function between robot~$i$ and desired point $k$, and $\mathcal T_{i,k}^{[1]}$ in~\eqref{task_point_i} only contains a point $\bm\sigma_i=\bold{x}_k^{d}$. Substituting the position $\bold{x}_i$ of the $i$-th robot for $\bm\sigma_i$ into $\mathcal T_{i,k}^{[1]}$ in~\eqref{task_point_i}, the characterizing function of the $k$-th desired-point convergence task $\mathcal T_{i,k}^{[1]}$ becomes
\begin{align}
\label{err_point_i}
\phi_{i,k}^{[1]}(\bold{x}_i)=\|\bold{x}_i-\bold{x}_k^{d}\|, i\in\mathcal V.
\end{align}
It follows from~\eqref{robot_dynamic} and Definition~\ref{definition_minimum_energy} that the CBF constraint of desired-point convergence is
\begin{align}
\label{task_1_CBF_constraint}
&\frac{\partial\phi_{i,k}^{[1]}(\bold{x}_i)}{\partial  \bold{x}_i\t} (\bold{u}_i-\bold{\widehat{v}}_i)+\gamma\big(\phi_{i,k}^{[1]}( \bold{x}_i)\big)\nonumber\\
&-(1-\alpha_{i,k}) \varpi-\delta_{i,k}\leq 0, \;k\in\mathbb{Z}_1^N,
\end{align}
where $\delta_{i,k}\in\mathbb{R}^+$ denotes the slack variable, $\varpi\in\mathbb{R}^+$ is the penalty term for $\alpha_{i,k}=0$, $\widehat{\bold{v}}_i$ represents the $i$-th velocity estimator for $\bold{v}_d$ in \textbf{A6}, and $\gamma(\cdot)$ is the extended class $\mathcal K$ function in~Definition~\ref{definition_minimum_energy}.

Suppose the neighboring collision-avoidance task $\mathcal T_{i,j}^{[2]}$ between robot $i$ and the neighboring robot~$j, j\in\mathcal N_i,$  is 
\begin{align}
\label{task_collision_i}
\mathcal T_{i,j}^{[2]}=&\big\{ \bm\sigma_i:=[\sigma_{i,1}, \cdots, \sigma_{i,n}]\t\in\mathbb{R}^n\nonumber\\
				&~\big|~\phi_{i,j}^{[2]}( \bm\sigma_i)=r-\|\bm\sigma_i-\bold{x}_j\|\leq0\big\}
\end{align}
with $r$ given below~\eqref{neighbor_set}.
Replacing $\bm\sigma_i$ by the $i$-th robot's position $x_i$ in $\phi_{i,j}^{[2]}(\bm\sigma_i)$, one has that $\bold{x}_i\in \mathcal T_{i,j}^{[2]}$ if and only if $\|\bold{x}_i-\bold{x}_j\|\geq r$. Analogously, we can define the collision-avoidance CBF constraint to be 
\begin{align}
\label{task_2_CBF_constraint}
\frac{\partial\phi_{i,j}^{[2]}(\bold{x}_i)}{\partial \bold{x}_i\t} (\bold{u}_i-\widehat{\bold{v}}_i)+\gamma\big(\phi_{i,j}^{[2]}(\bold{x}_i)\big)\leq0,\;j\in\mathcal N_i(t).\
\end{align}
Since the constraint (\ref{task_2_CBF_constraint}) does not contain slack variables, one has that the collision avoidance of $\|\bold{x}_{i}(t)-\bold{x}_{j}(t)\|>r$ can be rigorously guaranteed via the forward-invariance property in~\eqref{invariant_property} later. For non-neighboring robot $j\notin \mathcal N_i$, it follows from $\mathcal N_i$ in \eqref{neighbor_set} that the collision avoidance is satisfied.

Suppose the obstacle-avoidance task $\mathcal T_{i,l}^{[3]}$ between robot $i, i\in\mathcal V,$ and obstacle $l, l\in\mathcal M$ is described by
\begin{align}
\label{task_obstacle_i}
\mathcal T_{i,l}^{[3]}=&\big\{ \bm\sigma_i:=[\sigma_{i,1}, \cdots, \sigma_{i,n}]\t\in\mathbb{R}^n\nonumber\\
				&~\big|~\phi_{i,l}^{[3]}( \bm\sigma_i):=r+r_l^o-\|\bm\sigma_i-\bold{x}_l^o\|\leq0\big\},
\end{align}
where $r$ is given below~\eqref{neighbor_set}, and $r_l^o, \bold{x}_l^o$ are defined in \eqref{obstacle_boundary}, respectively. The characterzing function $\phi_{i,l}^{[3]}(\sigma_i, \bold{x}_l^o)\le 0$ represent all points which keep a safe distance $r$ to obstacle $l$. Then, the corresponding obstacle-avoidance CBF constraint becomes 
\begin{align}
\label{task_3_CBF_constraint}
\frac{\partial\phi_{i,l}^{[3]}(\bold{x}_i, \bold{x}_l^o)}{\partial \bold{x}_i} (\bold{u}_i-\widehat{\bold{v}}_i)+\gamma\big(\phi_{i,l}^{[3]}( \bold{x}_i, \bold{x}_l^o)\big)\leq0, l\in\mathcal O,
\end{align}
which also guarantees collision avoidance.

\subsection{Constrained Optimization}
In this subsection, combining the allocation conditions and CBF constraints in Eqs.~\eqref{eq_con_a_rob_a_poi}, \eqref{eq_con_a_poi_a_rob},~\eqref{task_1_CBF_constraint}, \eqref{task_2_CBF_constraint}, \eqref{task_3_CBF_constraint} together, we formulate the {\it MPCM navigation in obstacle environments} to be a constrained optimization problem for robot~$i$, below, 
\begin{subequations}
\label{resilient_prioritization_optimization}
\begin{align}
&\min\limits_{\bold{u}_i,\bm\delta_{i}, \bm\alpha} b\|\bold{1}_N\t \bm\alpha -\bold{1}_N\t \|^2+c\|\bm\delta_{i}\|^2+\|\bold{u}_i-\widehat{\bold{v}}_i\|^2\\
\mathrm{s.t.}~&\mathrm{Eq}.~\eqref{task_1_CBF_constraint}: \mathrm{Desired~point~convergence}, \\
&\mathrm{Eq}.~\eqref{task_2_CBF_constraint}: \mathrm{Neighboring~collision~avoidance},\\
&\mathrm{Eq}.~\eqref{task_3_CBF_constraint}: \mathrm{Obstacle~avoidance},\\
&\mathrm{Eq}.~\eqref{eq_con_a_rob_a_poi}: \mathrm{Robot~allocation~(i.e.,}~ \bm\alpha_i\t\mathbf{1}_N=1),\\
&\mathrm{Eq}.~\eqref{robot_dynamic}: \mathrm{Input~limitation~(i.e.,}~\|\bold{u}_i\|\leq u_{\max})
%&~~~~~~~~~~~~~~~~~~~~~~~~\forall i\in\mathcal V,\notag
%a&=b+c\notag\\
\end{align}		
\end{subequations}
where $b,c\in\mathbb{R}^+$ are the weights for different cost terms in~(\ref{resilient_prioritization_optimization}a), $\bm\alpha=[\bm\alpha_1, \cdots, \bm\alpha_N]\t$ are the allocation matrix with the $i$-th column vector $\bm\alpha_i$ given in~Section~\ref{subsec_allo}, $\bm\delta_i=[\delta_{i,1}, \dots, \delta_{i,N}]\t\in\mathbb{R}^{N+}$ denotes the slack vector, $\bold{u}_i$ is the input of robot $i$ in~\eqref{robot_dynamic}, $\widehat{\bold{v}}_i$ is the $i$-th velocity estimator for $\bold{v}_d$ given in~\eqref{task_1_CBF_constraint}. 
The cost function (\ref{resilient_prioritization_optimization}a) contains three terms, namely, the desired-point allocation error $(\bold{1}_N\t \bm\alpha -\bold{1}_N\t)$ in~\eqref{eq_con_a_poi_a_rob}, the input error $\bold{u}_i-\widehat{\bold{v}}_i$, and the slack vector $\|\bm\delta_{i}\|$ for $N$ desired-point convergence tasks simultaneously. Note that the weights $b, c$ in (\ref{resilient_prioritization_optimization}a) are set to be $b\gg c\gg 1$ to describe the importance of the three terms, which will be utilized to guarantee that one point is tracked by only one robot later.

Note that the desired-point allocation errors $(\bold{1}^T_N\bm\alpha-\bold{1}^T_N)$ in \eqref{resilient_prioritization_optimization} are difficult to encode as constraints for each robot $i, i\in\mathcal V$ to execute. Particularly, since the CATE optimization \eqref{resilient_prioritization_optimization} and all the constraints (\ref{resilient_prioritization_optimization}b)-(\ref{resilient_prioritization_optimization}f)  are executed separately for each robot, one has that any additional constraints must be encoded to be the form that each robot can execute as well. However, it follows from $\bm\alpha=[\bm\alpha_1, \cdots, \bm\alpha_N]\t$ and  $\bm\alpha_i=[\alpha_{i,1}, \cdots, \alpha_{i, N}]\t$ that such errors $(\bold{1}^T_N\bm\alpha-\bold{1}^T_N)$  can only be separated for each desired point rather than each robot, i.e., $\sum_{j=1}^N\alpha_{i,j}=1, i\in\mathcal V$, which thus cannot be encoded as a constraint for each robot. As a remedy, it is a common approach to formulate such complex desired-point allocation errors as cost terms, which can avoid the complexity of nonlinear constraints and simplify the expression \cite{ben2001lectures}. Moreover, by treating $(\bold{1}^T_N\bm\alpha-\bold{1}^T_N)$ as a cost term, the weights $b, c$ in \eqref{resilient_prioritization_optimization} can be flexibly adjusted, allowing for better balance and prioritization of different objectives, thereby improving the efficiency of the solution. Additionally, such a design in the cost term leverages the stronger adaptability of optimization solvers, such as Gurobi \cite{gurobi}, in handling objective functions, which avoids instability and enhances robustness.

If the $k$-th desired point is allocated to the $i$-th robot, then the constraints~(\ref{resilient_prioritization_optimization}e) indicate that $\alpha_{i,k}=1$ and $\alpha_{i,q}=0, \forall q\neq k \in\mathbb{Z}_1^N$, it follows from \eqref{task_1_CBF_constraint} that $(1-\alpha_{i,k})\varpi=0$ and $(1-\alpha_{i,q})\varpi=\varpi$, for $q \ne k$, which implies that (\ref{resilient_prioritization_optimization}b) can be divided into two subgroups: the $k$-th one is the most strict, while the remaining $N-1$ ones are more relax because of the large penalty $\varpi$, i.e.,
\begin{align}
\label{resilient_prioritization_constraint_1}
&\frac{\partial\phi_{i,k}^{[1]}( \bold{x}_i)}{\partial \bold{x}_i\t} ( \bold{u}_i-\widehat{\bold{v}}_i)+\gamma\big(\phi_{i,k}^{[1]}(\bold{x}_i)\big)-\delta_{i,k}\leq 0,\nonumber\\
&\frac{\partial\phi_{i,q}^{[1]}(\bold{x}_i)}{\partial \bold{x}_i\t} (\bold{u}_i-\widehat{\bold{v}}_i)+\gamma\big(\phi_{i,q}^{[1]}( \bold{x}_i)\big)-\delta_{i,q}-\varpi\leq 0, \forall q \ne k.
%& \forall q\neq k\in\mathbb{Z}_1^N.
\end{align}
In this way, if $\delta_{i,k}, k\in\mathbb{Z}_1^N,$ are minimized to be zeros,  it follows from the asymptotic-convergence property in~\eqref{convergence_property} that robot $i$ only converges to the allocated $k$-th desired point $\bold{x}_k^d$.

\begin{algorithm}
\caption{CATE Optimization for {\it MPCM Navigation in Obstacle Environments}}
\label{algorithm1}
\KwData{} 
The initial states of robot $i, i\in\mathcal V$: $\bold{x}_i$ in \eqref{robot_dynamic}\; 
The states of desired points: $\bold{x}_i^d, \bold{v}_d$ in \eqref{desired_velocity}\;
The center and radius of obstacles: $r_l^o, \bold{x}_l^o, l\in\mathcal M$ in~\eqref{obstacle_boundary}\;
Task characterizing functions: $\phi_{i,k}^{[1]}, k\in\mathbb{Z}_1^N, \phi_{i,j}^{[2]}, j\in\mathcal N_i, \phi_{i,l}^{[3]}, l\in\mathcal M$ in \eqref{err_point_i}, \eqref{task_collision_i} and \eqref{task_obstacle_i}\;
An arbitrary initial allocation matrix: $\bm\alpha$ in \eqref{eq_con_a_rob_a_poi}\;
Parameters: $b,c, u_{\max}, \varpi, r, R$\;               
                       
\KwResult{Optimal allocation matrix $ \bm\alpha^{\ast}$, slack variable $ \bm\delta_{i}^{\ast}$ and control input $\bold{u}_i^{\ast}$ for robot $i$}
%Initialize time step $r\leftarrow0$\;
%Initialize $\delta_{i,k}, \alpha_i$\;
%not converged and $r\leq T$

\For{$i\in\mathcal V$ }{ $ \widehat{\bold{v}}_i, \forall i\in\mathcal V \leftarrow$ get the $i$-th velocity estimator; 

%\For {each $W^e \in W$}{
%                Update $W^e$ with the rule defined in\ref{equ:W}\;
%            }  

$\bold{x}_i, \forall i\in\mathcal V\leftarrow$ get the $i$-th robot's position\; 

Calculate the cost function (\ref{resilient_prioritization_optimization}a)\;

Calculate the constraints (\ref{resilient_prioritization_optimization}b)-(\ref{resilient_prioritization_optimization}f)
$\leftarrow$ $ \phi_{i,k}^{[1]}(\bold{x}_i), \phi_{i,j}^{[2]}(\bold{x}_i), \phi_{i,l}^{[3]}(\bold{x}_i), \bm\alpha$%$ \widehat{v}_i, x_i, x_j, j\in\mathcal N_i, x_l^o, r_l^o, \varpi$\;

Solve the optimization problem~\eqref{resilient_prioritization_optimization} by the Gurobi solver tools~\cite{gurobi}\;

Calculate $\bm\alpha^{\ast}, \bm\delta_{i}^{\ast}, \bold{u}_i^{\ast}, \forall i\in\mathcal V$\;

%Solve the problem in \eqref{resilient_prioritization_optimization} f
Execute $\bold{u}_i^{\ast}, \forall i\in\mathcal V$ for {\it MPCM navigation}\;
}
%Get a spontaneous spatial sequence: $s[1], \cdots, s[N]\leftarrow $ $ \alpha_i^{\ast},$ $ i\in\mathcal V$.
\end{algorithm}
 
The main workflow of the proposed CATE optimization~\eqref{resilient_prioritization_optimization} is described in Algorithm \ref{algorithm1}. Before executing the {\it MPCM} navigation, the CATE initializes quantities, such as the states of robots $\bold{x}_i$ (line 1), the desired points $\bold{x}_i^d, \bold{v}_d$ (line 2), the obstacles $r_l^o, \bold{x}_l^o$  (line 3), the allocation matrix $\bm\alpha$  (line 5) and the cost weights $b,c$ (line 6). For each robot $i\in\mathcal V$, if only a small proportion of robots have access to the desired velocity $\bold{v}_d(t)$, the CATE algorithm first calculates the distributed estimator $\bold{\widehat{v}}_i$ for each robot $i$, i.e., $\lim_{t\rightarrow\infty}\bold{\widehat{v}}_i(t)=\bold{v}_d(t)$ (line 8). Then, the corresponding soft cost functions and hard constraints in \eqref{resilient_prioritization_optimization} are established (lines 9-11). The constrained optimization problem \eqref{resilient_prioritization_optimization} is solved through the Gurobi solver to get the optimal allocation matrix, slack vector and inputs $\bm\alpha^{\ast}, \bm\delta_i^{\ast}, \bold{u}_i^{\ast}$ of each robot (lines 12-13).  Finally, the optimal inputs $\bold{u}_i^{\ast}, i\in\mathcal V$ are executed for each robot $i, i\in\mathcal V$ (line 14). For the next time instant of the new states of elements, the CATE optimization \eqref{resilient_prioritization_optimization} for each robot in Algorithm~\ref{algorithm1} will be executed repeatedly for the new optimal $\bm\alpha^{\ast}, \bm\delta_i^{\ast}, \bold{u}_i^{\ast}$.

\begin{remark}
\label{remark_difference}
Unlike the previous works~\cite{sabattini2017optimized,alonso2012image,turpin2014capt} in Fig.~\ref{previous_PCM}, which calculate the optimal allocation matrix $\bm\alpha^{\ast}$ based on the minimal squared-distance reassignment between robots and designed points in advance, the proposed CATE algorithm~\eqref{resilient_prioritization_optimization} concurrently calculates the allocation matrix $\bm\alpha^{\ast}$ and inputs $\bold{u}_i^{\ast}$ directly by solving the constrained optimization because all the task elements (including robots, designed points and obstacles) are encoded in the CBF constraints in Eqs.~(\ref{resilient_prioritization_optimization}b)-(\ref{resilient_prioritization_optimization}d) and allocation constraints are simply restricted to be $N!$ candidates of one-to-one mappings. Therefore, the CATE algorithm is reactive to all the elements and can achieve {\it MPCM navigation in obstacle environments}.

%is reactive to the environment elements (such as robots, designed points and obstacles) as CBF constraints, which concurrently calculates the allocation matrix $\bm\alpha^{\ast}$ and inputs $u_i^{\ast}$
%by solving the constrained optimization problem \eqref{resilient_prioritization_optimization}. It is thus can work for PCM navigation in obstacle environments.
\end{remark}

\begin{remark}
\label{remark_moving_obs}
For the more challenging and interesting scenario of dynamic obstacles, which has been extensively explored in prior works \cite{mahulea2017robot, alonso2019distributed,hu2021decentralized,ma2017overview}, 
the proposed CATE algorithm \eqref{resilient_prioritization_optimization} can also deal with it by conveniently adding the velocities of the obstacles to the corresponding obstacle-avoidance CBF constraints \eqref{task_3_CBF_constraint}, which is similar to many previous CBF works \cite{wang2017safety,hu2024ordering}. Precisely, suppose the velocity of dynamic obstacle $l$ is $\bold{v}_l^o, l\in\mathcal O$, then we can rewrite the CBF constraint \eqref{task_3_CBF_constraint} to be
%\begin{align*}
%\frac{\partial\phi_{i,l}^{[3]}(\bold{x}_i, \bold{x}_l^o)}{\partial  \bold{x}_i} ( \bold{u}_i-\widehat{\bold{v}}_i)+\gamma\big(\phi_{i,l}^{[3]}( \bold{x}_i, \bold{x}_l^o)\big)\leq0, l\in\mathcal O.
%\end{align*}
$\frac{\partial\phi_{i,l}^{[3]}(\bold{x}_i, \bold{x}_l^o)}{\partial  \bold{x}_i} ( \bold{u}_i-\widehat{\bold{v}}_i-\bold{v}_l^o)+\gamma\big(\phi_{i,l}^{[3]}( \bold{x}_i, \bold{x}_l^o)\big)\leq0, l\in\mathcal O.$
In this way, the proposed CATE algorithm \eqref{resilient_prioritization_optimization} can be seamlessly used in dynamic obstacle environments.
However, for the scenario of obstacles with random unknown velocities, it is still a challenging problem for {\it MPCM navigation}, which will be investigated in future work. Moreover, additional simulations of dynamic obstacles will be conducted to validate its effectiveness in Fig.~\ref{moving_obstacle} later.
\end{remark}

\begin{remark}
The online feature of the CATE optimization algorithm \eqref{resilient_prioritization_optimization} has four advantages. Firstly, the online CATE algorithm is reactive to unpredictable obstacles (satisfying a mild assumption \textbf{A4}) and acts as the feedback control directly without requiring a path (re)-planning process. Secondly, the online CATE algorithm can be conveniently utilized in dynamic obstacle environments by adding obstacle velocities to the obstacle-avoidance CBF constraints in Remark~\ref{remark_moving_obs}. Thirdly, the online CATE algorithm is robust to sudden changes in environments, whereas offline path planning may require time-consuming and low-efficiency recalculations. Last but not least, despite the fixed obstacles and desired points, calculating the entire optimal path remains challenging and time-consuming because the obstacles occupy the space between the robots and the desired points. However, the online CATE algorithm iteratively finds the optimal solution for each subsequent step, making it more efficient and flexible.
\end{remark}

\section{Features and Convergence Analysis}
\label{sec_main}
%In this section, we will introduce different features of the CATE algorithm, and then decouple the algorithm to guarantee the properties of {\bf P1-P4} for the {\it MPCM navigation in obstacle environments}.

\subsection{Algorithm Features}
\label{sub_algorithm_feature}

%The constraint (\ref{resilient_prioritization_optimization}e) denotes the maximum of $u_i$ in~(\ref{resilient_prioritization_optimization}a).
\subsubsection{Feasibility guarantee}
\label{subsub_feasibility}
Recalling {\bf A3}-{\bf A4}, we ensure that there always exists a trivial feasible solution for robot $i$ in the constrained optimization problem in (\ref{resilient_prioritization_optimization}), i.e.,
\begin{align}
\label{trivial_solution}
\big\{& \bold{u}_i=\widehat{\bold{v}}_i, \delta_{i,k} \mbox{ is sufficiently large},\nonumber\\ 
& \bm\alpha_i=[0,\dots,\underbrace{1}_{i\mathrm{-th}},\dots,0]\in\mathbb{R}^N\big\},
\end{align}
which satisfy all the constraints in (\ref{resilient_prioritization_optimization}b)-(\ref{resilient_prioritization_optimization}e) due to the arbitrary selection of the slack variable $\delta_{i,k}$.

%\begin{remark}
%The feasibility under input constraints
%\end{remark}

%For the feasibility of the CATE algorithm (\ref{resilient_prioritization_optimization}), there always exists a feasible solution due to Assumption~\ref{assum_robot_distance} and the use of the slack variable $\delta_{i,k}$. It is obvious that a trivial feasible solution is 
%%$u_i=0, \delta_{i,k}$ is a sufficiently large value and $\alpha_i=[0,\dots,1,\dots,0]\in\mathbb{R}^N$ with $1$ being the $k$-th entry. 
%
%with $1$ being the $i$-th entry.

\subsubsection{One-to-one correspondence} From the hard constraint of robot allocation in (\ref{resilient_prioritization_optimization}e), one can first ensure that one robot only tracks one desired point, i.e.,~\eqref{eq_con_a_rob_a_poi}. 
To achieve the one-to-one correspondence between robots and desired points in Section~\ref{second_block}, we will prove that one desired point is tracked by only one robot in the following lemma.

\begin{lemma}
For the soft cost term $\{\bold{1}_N\t \bm\alpha -\bold{1}_N\t\}$~in (\ref{resilient_prioritization_optimization}a), the CATE algorithm (\ref{resilient_prioritization_optimization}) can guarantee that one desired point can be tracked by only one robot,  i.e., the condition in~\eqref{eq_con_a_poi_a_rob}.
\end{lemma}

\begin{proof}
We can prove it by contradiction.
Firstly, we assume that there exists a finite time $T>0$ such that one desired point is tracked by one robot for $t\in[0, T)$ but not $t=T$, which implies that 
at least two robots track one desired point at $t=T$, i.e., $b\| \bold{1}_N\t \bm\alpha(T) -\bold{1}_N\t \|^2\geq 2b.$
%\begin{align*}
%%\label{one_p_one_r_cond0}
%b\|{\blue \bold{1}_N\t \bm\alpha(T) -\bold{1}_N\t} \|^2\geq 2b.
%\end{align*}
During the time interval $t\in[0,T)$, one has that the cost term in \eqref{resilient_prioritization_optimization}
\begin{align}
\label{one_p_one_r_cond1}
b\|\bold{1}_N\t \bm\alpha(t) -\bold{1}_N\t \|^2=0, \forall t\in[0, T).
\end{align} Meanwhile, since $\bold{u}_i, \widehat{\bold{v}}_i$ are bounded by $u_{\max}$, and the initial value of slack variables $\bm\delta_{i}(0)$ are also bounded, 
one has that there always exists a constant $\varsigma\in\mathbb{R}^+$ such that 
\begin{align}
\label{one_p_one_r_cond2}
c\|\bm\delta_{i}(t)\|^2+\|\bold{u}_i(t)-\widehat{\bold{v}}_i(t)\|^2<\varsigma, \forall t\in[0, T).
\end{align}
Let $\Omega_i:=b\|\bold{1}_N\t \bm\alpha -\bold{1}_N\t \|^2+c\|\bm\delta_{i}\|^2+\|\bold{u}_i-\widehat{\bold{v}}_i\|^2$ be the cost function (\ref{resilient_prioritization_optimization}a). Then, it follows from~\eqref{one_p_one_r_cond1} and \eqref{one_p_one_r_cond2} that 
%\begin{align}
%\label{one_p_one_r_cond3}
$\Omega_i(t)<\varsigma, \forall t\in[0, T)$, 
%\end{align}
which implies that 
\begin{align}
\label{one_p_one_r_cond3}
\Omega_i(T)<\varsigma
\end{align}
as well due to the optimization process. However, recalling~\eqref{one_p_one_r_cond1}, if the weight $b$ in (\ref{resilient_prioritization_optimization}a) is selected to be sufficiently large such that $b>{\varsigma}/{2}$, one has that 
\begin{align}
\label{one_p_one_r_cond4}
b\|\bold{1}_N\t \bm\alpha(T) -\bold{1}_N\t \|^2>\varsigma.
\end{align}
Combining with the fact $c\|\bm\delta_{i}\|^2\geq0, \|\bold{u}_i-\widehat{\bold{v}}_i\|^2\geq0$, it follows from~\eqref{one_p_one_r_cond4} that $\Omega_i(T)>\varsigma$, which contradicts the condition in~\eqref{one_p_one_r_cond3}. It indicates that one desired point is tracked by only one robot during the process, i.e.,~\eqref{eq_con_a_poi_a_rob} is guaranteed.
\end{proof}

\subsubsection{Livelock prevention} 
Due to the special structure of the CATE algorithm (\ref{resilient_prioritization_optimization}), we will prove in the following lemma that it can eventually prevent the undesirable livelock phenomenon of indefinitely frequent reallocation between the desired points and robots.
\begin{lemma}
\label{live_lock}
The proposed CATE algorithm (\ref{resilient_prioritization_optimization}) guarantees that there will be only one optimal invariant allocation matrixes $\bm\alpha^{\ast}$ satisfying $\|\bold{1}_N\t \bm\alpha^{\ast}-\bold{1}_N\t \|=0$ eventually. 
\end{lemma}

\begin{proof}
We will prove by contradiction. Firstly, we assume that there exist a sufficiently large time $T\in\mathbb{R}^+$ by which the {\it MPCM navigation in obstacle environments} has already been achieved and two different optimal allocation matrices $\bm{\alpha}^{1\ast}, \bm{\alpha}^{2\ast}, \bm{\alpha}^{1\ast}\neq \bm{\alpha}^{2\ast}$ satisfying
\begin{align}
\label{right_allocation}
&\|\bold{1}_N\t \bm\alpha^{1\ast}(t)-\bold{1}_N\t \|=0, t<T, \nonumber\\
&\|\bold{1}_N\t \bm\alpha^{2\ast}(t)-\bold{1}_N\t \|=0,  t=T,
\end{align}
which implies that the optimal allocation matrix $\bm\alpha^{1\ast}$ is changed to $\bm\alpha^{2\ast}$ at $t=T$. 

Given the condition of $\bm\alpha^{1\ast}$ in \eqref{right_allocation}, without loss of generality, let $k$ be the allocated $k$-th desired point for robot $i$, one has that $\alpha_{i,k}=1, \alpha_{i,q}=0, \forall q\neq k\in\mathbb{Z}_1^N$ and the corresponding CATE algorithm (\ref{resilient_prioritization_optimization}) for robot $i$ degenerates into the following minimal-energy task execution, i.e., 
\begin{subequations}
\label{resilient_prioritization_optimization1}
\begin{align}
\min\limits_{\bold{u}_i,\bm\delta_i}&~c\|\bm\delta_{i}\|^2+\|\bold{u}_i-\widehat{\bold{v}}_i\|^2\\
\mathrm{s.t.}~&\bold{g}_{i}^{[1]}\leq \mathbf{0}_N, \bold{g}_{i}^{[2]}\leq \mathbf{0}_{m_i}, \bold{g}_{i}^{[3]}\leq \mathbf{0}_{M}, g_{i}^{[4]}\leq0,
\end{align}		
\end{subequations}
where $N\in\mathbb{Z}^+, m_i\in\mathbb{Z}^+, M\in\mathbb{Z}^+$ are the number of robots, neighboring robots (i.e., $m_i=|\mathcal N_i |$) in~\eqref{neighbor_set} and obstacles, respectively,  and the constraints $\bold{g}_{i}^{[1]}\in\mathbb{R}^{N}, \bold{g}_{i}^{[2]}\in\mathbb{R}^{m_i}, \bold{g}_{i}^{[3]}\in\mathbb{R}^{M}, g_{i}^{[4]}\in\mathbb{R}$ are $\bold{g}_{i}^{[1]}:=\frac{\partial\Phi_{i}^{[1]}}{\partial \bold{x}_i\t} (\bold{u}_i-\widehat{\bold{v}}_i)+\gamma\big(\Phi_{i}^{[1]}\big)-\bold{D}_i \varpi-\bm\delta_{i}, \bold{g}_{i}^{[2]}:=\frac{\partial\Phi_{i}^{[2]}}{\partial \bold{x}_i\t} (\bold{u}_i-\widehat{\bold{v}}_i)+\gamma\big(\Phi_{i}^{[2]}\big), \bold{g}_{i}^{[3]}:=\frac{\partial\Phi_{i}^{[3]}}{\partial \bold{x}_i\t} (\bold{u}_i-\widehat{\bold{v}}_i)+\gamma\big(\Phi_{i}^{[3]}\big), g_{i}^{[4]}:=\|\bold{u}_i\|-u_{\max}.$
%\begin{align}
%\label{vector_constrains}
%\bold{g}_{i}^{[1]}:=&\frac{\partial\Phi_{i}^{[1]}}{\partial \bold{x}_i\t} (\bold{u}_i-\widehat{\bold{v}}_i)+\gamma\big(\Phi_{i}^{[1]}\big)-\bold{D}_i \varpi-\bm\delta_{i},\nonumber\\
%\bold{g}_{i}^{[2]}:=&\frac{\partial\Phi_{i}^{[2]}}{\partial \bold{x}_i\t} (\bold{u}_i-\widehat{\bold{v}}_i)+\gamma\big(\Phi_{i}^{[2]}\big), \nonumber\\
%\bold{g}_{i}^{[3]}:=&\frac{\partial\Phi_{i}^{[3]}}{\partial \bold{x}_i\t} (\bold{u}_i-\widehat{\bold{v}}_i)+\gamma\big(\Phi_{i}^{[3]}\big),\nonumber\\
%g_{i}^{[4]}:=&\|\bold{u}_i\|-u_{\max}.
%\end{align}
Here, $\gamma(\cdot)$ is evaluated entrywise, and $ \bold{D}_i:=[1, \cdots, 0, \cdots,1] \in\mathbb{R}^N, \Phi_{i}^{[1]}:=[\phi_{i,1}^{[1]},\cdots, \phi_{i,N}^{[1]}]\t\in\mathbb{R}^{N}, \gamma(\Phi_i^{[1]}):=[\gamma(\phi_{i,1}^{[1]}), \dots, \gamma(\phi_{i,N}^{[1]})]\t\in\mathbb{R}^{N}, \Phi_i^{[2]}:=[\phi_{i,j}^{[2]},\cdots, \phi_{i,l}^{[2]}]\t\in\mathbb{R}^{m_i}, \gamma(\Phi_i^{[2]}):=[\gamma(\phi_{i,j}^{[2]}), \dots,$ $ \gamma(\phi_{i,l}^{[2]})]\t\in\mathbb{R}^{m_i}, \Phi_i^{[3]}:=[ \phi_{i,1}^{[3]},$ $\cdots, \phi_{i,M}^{[3]}]\t\in\mathbb{R}^{M}, \gamma(\Phi_i^{[3]}):=[\gamma(\phi_{i,1}^{[3]}), \dots, \gamma(\phi_{i,M}^{[3]})]\t\in\mathbb{R}^{M}$ are the column vectors,
where $m_i=|\mathcal N_i|\in\mathbb{Z}^+$ is the cardinality of the neighbor set $\mathcal N_i$, and $0$ in $\bold{D}_i$ is at the $k$-th entry such that $\alpha_{i,k}=1$.

Meanwhile, recalling the condition that the {\it MPCM navigation} has already been achieved at $t=T^{-}$ with $T^{-}$ being the left limit of $T$, one has that 
\begin{align}
\label{livelock_con1}
\bold{x}_i(T^{-})=\bold {x}_k^d(T^{-}), \bold{u}_i(T^{-})=\bold{\widehat{v}}_i(T^{-})=\bold{v}_d(T^{-}),
\end{align}
which follows from \eqref{err_point_i} that $\phi_{i,k}^{[1]}(\bold{x}_i(T^{-}))=0$. Then, it follows from the constraints \eqref{resilient_prioritization_constraint_1} that 
\begin{align}
\label{livelock_con2}
\delta_{i,k}(T^{-})=0, \delta_{i,q}(T^{-})=0, \forall q\neq k\in\mathbb{Z}_1^N,
\end{align}
i.e., $\bm\delta_i=\mathbf{0}$. Substituting \eqref{right_allocation}, \eqref{livelock_con1} and \eqref{livelock_con2} into \eqref{resilient_prioritization_optimization1} yields the cost function $\Omega_i(T^{-})=b\|\bold{1}_N\t \bm\alpha^{1\ast}(T^{-})-\bold{1}_N\t\|+c\|\bm\delta_{i}(T^{-})\|^2+\|\bold{u}_i(T^{-})-\widehat{\bold{v}}_i(T^{-})\|^2=0$ with $\Omega_i$ given in \eqref{one_p_one_r_cond3}. 

When $t=T^{-}$ changes to $t=T$, together with \eqref{right_allocation} and \eqref{livelock_con1}, one has that the optimization (\ref{resilient_prioritization_optimization1}) still remains the same with $\bold{u}_i(T)=\bold{\widehat{v}}_i(T)=\bold{v}_d(T)$. Given the different $\bm\alpha^{2\ast}$, there exists a different desired point $j\neq k$ satisfying $\alpha_{i,j}=1, \alpha_{i,q}=0, \forall q\neq j\in\mathbb{Z}_1^N$. Analogously, in order to satisfy the constraints \eqref{resilient_prioritization_constraint_1}, one has that 
\begin{align}
\label{livelock_con3}
\delta_{i,j}(T)>0, \delta_{i,q}(T)=0, \forall q\neq j\in\mathbb{Z}_1^N,
\end{align}
because of $\bold{x}_i(T)\neq\bold {x}_j^d(T)$ in \eqref{livelock_con1}, which implies that $\bm\delta_i\neq\mathbf{0}$. Then, combining \eqref{right_allocation} and \eqref{livelock_con3} together yields $\Omega_i(T)>0=\Omega_i(T^{-})$. However, such an increment of $\Omega_i$ at $t=T$ contradicts the minimization of the CATE algorithm~(\ref{resilient_prioritization_optimization}). The proof is thus completed.
\end{proof}

\subsection{Decoupling and Analysis of {\bf P1-P4} }
\label{sub_ana_P1_P4}
Recalling the feasibility guarantee in Section~\ref{subsub_feasibility}, the next step is to find the optimal solution $\{\bm\alpha^{\ast}, \bold{u}_i^{\ast}, \bm\delta_{i}^{\ast}\}$
to prove the convergence to the {\it MPCM navigation in obstacle environments}. However, the existence of the integer constraints (\ref{resilient_prioritization_optimization}e) may make the CATE algorithm hard to analyze.
To simplify the convergence analysis, we will decouple the CATE algorithm~\eqref{resilient_prioritization_optimization} into task allocation and explicit minimum-energy task execution 
by the following proposition.

\begin{proposition}
\label{new_proposition}
The allocation matrix $\bm\alpha, i\in\mathcal V$, in~the CATE algorithm \eqref{resilient_prioritization_optimization} will finally converge to the optimal $\bm\alpha^{\ast}$ and keep invariant, i.e., $\lim_{t\rightarrow T_1}\bm\alpha(t)=\bm\alpha^{\ast},$
with a constant time $T_1\in\mathbb{R}^+$, such that $\|\bold{1}_N\t \bm\alpha^{\ast}(t)-\bold{1}_N\t \|=0$ for $ t \ge T_1$.
\end{proposition}

\begin{proof}
According to Remark \ref{remark_difference}, there are $N!$ candidates for the one-to-one robot-point mapping, i.e.,  
\begin{align}
\label{remark_two_conditions}
\|\bold{1}_N\t \bm\alpha-\bold{1}_N\t\|=0, \bm\alpha_i\t\bold{1}_N=1, i\in\mathcal V,
\end{align}
which implies that there always exists a finite time $T_1\in\mathbb{R}^{+}$ such that the optimal $\bm\alpha^{\ast}$ satisfying \eqref{remark_two_conditions} can be found via Gurobi using the efficient branch-and-bound, cutting-plane, and heuristic algorithms \cite{gurobi_2}, i.e.,
\begin{align}
\label{proposition_condition_1}
&\exists T_1\in\mathbb{R}^+,~\mathrm{such~that}~\|\bold{1}_N\t \bm\alpha^{\ast}(T_1)-\bold{1}_N\t\|=0,\nonumber\\
& (\bm\alpha_i^{\ast}(T_1))\t\bold{1}_N=1, \forall i\in\mathcal V.
\end{align}
Meanwhile, it follows from $b\gg c\gg 1$ in (\ref{resilient_prioritization_optimization}a) and Lemma~\ref{live_lock} that the dominating error $\|\bold{1}_N\t \bm\alpha-\bold{1}_N\t\|$ will be quickly minimized to zero by $\bm\alpha^{\ast}$ in \eqref{proposition_condition_1} according to the minimization of the CATE algorithm~(\ref{resilient_prioritization_optimization}), i.e., 
\begin{align}
\label{proposition_condition_2}
\lim_{t\rightarrow T_1}\|\bold{1}_N\t \bm\alpha^{\ast}(t)-\bold{1}_N\t\|=0.
\end{align} 
Combining \eqref{proposition_condition_1} and \eqref{proposition_condition_2}, one has that the CATE algorithm~\eqref{resilient_prioritization_optimization} will degenerate into the minimal-energy task execution~\eqref{resilient_prioritization_optimization1} if $t>T_1$, which can guide the robots towards the allocated desired points. The livelock scenario where the optimal $\bm\alpha^{\ast}$ is changed but still satisfies $\|\bold{1}_N\t \bm\alpha^{\ast}(t)-\bold{1}_N\t\|=0, \forall t>T_1$, is well prevented by Lemma~\ref{live_lock}. Therefore, there will eventually be only one optimal  $\bm\alpha^{\ast}$ satisfying $\|\bold{1}_N\t \bm\alpha^{\ast}(t)-\bold{1}_N\t \|=0$ for $ t \ge T_1$. The proof is thus completed.
\end{proof}

Based on~Proposition~\ref{new_proposition}, we can thus separate the convergence analysis of the CATE algorithm~\eqref{resilient_prioritization_optimization} into two time intervals of $t\in[0, T_1)$ and $t\in[T_1, \infty)$, respectively.

For $t\in[0, T_1)$, since the allocation vector $\bm\alpha_i$ keeps changing, one has that the desired points allocated for robots change as well, so are the two errors $\|\bm\delta_{i}\|^2, \|\bold{u}_i-\widehat{\bold{v}}_i\|^2$ in (\ref{resilient_prioritization_optimization}a), which is difficult to determine the unique desired-point convergence tasks. However, since the cost function~(\ref{resilient_prioritization_optimization}a) is bounded for $t\in[0, T_1)$, we ignore the evolution in $t\in[0,T_1)$, and focus on the convergence analysis when $t\in[T_1, \infty)$.

For $t\in[T_1, \infty)$, since $\bm\alpha^{\ast}$ are fixed and optimal in~Proposition~\ref{new_proposition}, one has that the integer constraints (\ref{resilient_prioritization_optimization}d) are eliminated. Furthermore, given the condition that $\|\bold{1}_N\t \bm\alpha^{\ast}-\bold{1}_N\t \|=0$, we can focus on analyzing the degenerated minimal-energy task execution \eqref{resilient_prioritization_optimization1} for the properties of {\bf  P3, P4, P2, P1} in order.

\subsubsection{Proof of {\bf P3, P4}}
 Recalling the forward invariance of Definition~\ref{definition_minimum_energy}, we first guarantee that the optimal $\bold{u}_i^{\ast}$ of degenerated optimization in \eqref{resilient_prioritization_optimization1} are locally Lipschitz continuous, 
which thus can inherits the forward invariance in~\eqref{invariant_property} to prove collision and obstacle avoidance in \textbf{P3, P4}.

\begin{itemize}
\item {\bf A7:} The first-order derivatives of $\bold{u}_i$ in \eqref{resilient_prioritization_optimization1} are assumed to be $\dot{\bold{u}}_i\leq\eta$ with an unknown constant $\eta\in\mathbb{R}^+$.
\end{itemize}

{\bf A7} is a necessary condition to prove the control inputs $u_i$ in  \eqref{resilient_prioritization_optimization1} are locally Lipschitz continuous in Lemma~\ref{Lemma_local_Lipschiz} later.

\begin{lemma}
\label{Lemma_local_Lipschiz}
For each robot $i, i\in\mathcal V$ governed by \eqref{robot_dynamic} and \eqref{resilient_prioritization_optimization1}, the $i$-th optimal solution of $\bold{u}_i^{\ast}$ is locally Lipschitz continuous in the target sets of $\mathcal T_{i,j}^{[2]}, \mathcal T_{i,l}^{[3]}$ in \eqref{task_collision_i} and \eqref{task_obstacle_i}, respectively.
\end{lemma}

\begin{proof}
Since the constrained optimization problem~\eqref{resilient_prioritization_optimization1} containing the input constraints (i.e., $g_{i}^{[4]}\leq0$) is difficult to analyze directly, we hereby separate the local-Lipschitz-continuous analysis of $\bold{u}_i$ in \eqref{resilient_prioritization_optimization1} into two parts based on $g_{i}^{[4]}\leq0$ being activated or not.

If the input constraints are activated (i.e., $g_{i}^{[4]}=0$), it follows from~\eqref{resilient_prioritization_optimization1} that $\|\bold{u}_i^{\ast}\|=u_{\max}$. Then, suppose there exists two different states $\bold{x}_1\neq \bold{x}_2\in\mathbb{R}^n$ satisfying $\|\bold{u}_i^{*}(\bold{x}_1)\|=\|\bold{u}_i^{*}(\bold{x}_2)\|=u_{\max}$, one has that $\exists E\in\mathbb{R}^+$ such that $\|\bold{u}_i^{*}(\bold{x}_1)-\bold{u}_i^{*}(\bold{x}_2)\|\leq E$. According to {\bf A7} and~\cite{ames2016control},  one can always find a positive constant $L\in\mathbb{R}^+$ such that $L\|\bold{x}_1-\bold{x}_2\|\geq \big\| \bold{u}_i^{*}(\bold{x}_1)-\bold{u}_i^{*}(\bold{x}_2)\big\|$, which thus implies that the local Lipschitz continuity of $\bold{u}_i^{\ast}$ are satisfied.

If the input constraints are not activated (i.e., $g_{i}^{[4]}<0$), then $g_{i}^{[4]}\leq0$ in \eqref{resilient_prioritization_optimization1} can be omitted. The optimization \eqref{resilient_prioritization_optimization1} becomes a control-lyapunov-function-control-barrier-function quadratic program (CLF-CBF QP) problem \cite{ames2016control}, namely, 
\begin{subequations}
\label{resilient_prioritization_optimization2}
\begin{align}
\min\limits_{\bold{u}_i,\bm\delta_i}&~c\|\bm\delta_{i}\|^2+\|\bold{u}_i-\widehat{\bold{v}}_i\|^2\\
\mathrm{s.t.}~&\bold{g}_{i}^{[1]}\leq \mathbf{0}_N, \bold{g}_{i}^{[2]}\leq \mathbf{0}_{m_i}, \bold{g}_{i}^{[3]}\leq \mathbf{0}_{M}. %g_{i}^{[4]}\leq 0
\end{align}
\end{subequations}
It follows from Eqs.~\eqref{task_collision_i} and \eqref{task_obstacle_i} that ${\partial \phi_{i,j}^{[2]}(\bold{x}_i)}/{\partial \bold{x}_i}={-(\bold{x}_i-\bold{x}_j)}/{\|\bold{x}_i-\bold{x}_j\|}\neq\mathbf{0}_2, \forall  \bold{x}_i\in\mathcal T_{i,j}^{[2]}$ and 
${\partial  \phi_{i,l}^{[3]}(\bold{x}_i)}/{\partial \bold{x}_i}$ $={-(\bold{x}_i-\bold{x}_l^o)}/{\| \bold{x}_i-\bold{x}_l^o\|}\neq\mathbf{0}_2, \forall x_i \in\mathcal T_{i,l}^{[3]}$
%\begin{align*}
%&\frac{\partial \phi_{i,j}^{[2]}(\bold{x}_i)}{\partial \bold{x}_i}=\frac{-(\bold{x}_i-\bold{x}_j)}{\|\bold{x}_i-\bold{x}_j\|}\neq\mathbf{0}_2, \forall  \bold{x}_i\in\mathcal T_{i,j}^{[2]},\nonumber\\
%&\frac{\partial  \phi_{i,l}^{[3]}(\bold{x}_i)}{\partial \bold{x}_i}=\frac{-(\bold{x}_i-\bold{x}_l^o)}{\| \bold{x}_i-\bold{x}_l^o\|}\neq\mathbf{0}_2, \forall x_i \in\mathcal T_{i,l}^{[3]},
%\end{align*}
are both locally Lipschitz continuous. One has that $\bold{u}_i^{\ast}$ in \eqref{resilient_prioritization_optimization2} are locally Lipschitz continuous according to (Theorem 3 in \cite{ames2016control}). The proof is almost the same as \cite{ames2016control} which is thus omitted here.
Combining two cases together yields $\bold{u}_i^{\ast}$ for \eqref{resilient_prioritization_optimization1} are locally Lipschitz continuous. The proof is thus completed.
\end{proof}

\begin{lemma}
\label{Lemma_P3_P4}
Under \textbf{A3, A4}, all the robots governed by \eqref{robot_dynamic} and \eqref{resilient_prioritization_optimization1} can guarantee the collision and obstacle avoidance all along, i.e., \textbf{P3, P4} are achieved.
\end{lemma}

\begin{proof}
Using \textbf{A3, A4}, one has that the initial states satisfy $\phi_{i,j}^{[2]}(\bold{x}_i(0))=r-\|\bold{x}_i(0)-\bold{x}_j(0)\|\leq0$, and $\phi_{i,l}^{[3]}(\bold{x}_i(0)):=r+r_l^o-\|\bold{x}_i(0)-\bold{x}_l^o\|\leq 0$.
Combining with Eqs.~\eqref{task_collision_i} and \eqref{task_obstacle_i} gives $\bold{x}_i(0)\in \mathcal T_{i,j}^{[2]}$ and $\bold{x}_i(0)\in \mathcal T_{i,l}^{[3]}$.
Based on Lemma~\ref{Lemma_local_Lipschiz}, one has that $\bold{u}_i^{\ast}$ in \eqref{resilient_prioritization_optimization1} are locally Lipschitz continuous in the target sets of $\mathcal T_{i,j}^{[2]}, \mathcal T_{i,l}^{[3]}$. Then, it follows from the forward invariance in \eqref{invariant_property} that $\bold{x}_i(t)\in \mathcal T_{i,j}^{[2]}, \bold{x}_i(t)\in \mathcal T_{i,l}^{[3]}, \forall t\geq 0$. The proof is thus completed.
\end{proof}

\subsubsection{Proof of {\bf P2}} Analogous to the proof of Lemma~\ref{Lemma_P3_P4}, we will first show that $\bold{u}_i^{\ast}$ in \eqref{resilient_prioritization_optimization1} are locally Lipschitz continuous in two cases and then prove \textbf{P2} using the asymptotic convergence in \eqref{convergence_property}.

\begin{lemma}
\label{lemma_final_convergence}
Under \textbf{A2} and \textbf{A6}, all the robots governed by \eqref{robot_dynamic} and~\eqref{resilient_prioritization_optimization1} can finally converge to the desired points and maneuver with $\bold{v}_d$ in \eqref{desired_velocity}, i.e., \textbf{P2} is achieved.
\end{lemma}

\begin{proof}
Analogously, if the input constraints are activated (i.e., $g_{i}^{[4]}=0$), one has that $\|\bold{u}_i^{\ast}\|=\bold{u}_{\max}$, which implies that the local Lipschitz continuity of $\bold{u}_i^{\ast}$ in \eqref{resilient_prioritization_optimization1} are satisfied as shown in the proof in Lemma~\ref{Lemma_local_Lipschiz}.

If the input constraints are not activated (i.e., $g_{i}^{[4]}<0$), one has that the optimization \eqref{resilient_prioritization_optimization1} becomes the CLF-CBF QP as shown in \eqref{resilient_prioritization_optimization2}. Then, from the fact that ${\partial \phi_{i,k}^{[1]}(\bold{x}_i)}/{\partial \bold{x}_i}={(\bold{x}_i-\bold{x}_k^{d})}/{\|\bold{x}_i-\bold{x}_k^{d}\|}\neq\mathbf{0}_2, \forall \bold{x}_i\notin\mathcal T_{i,k}^{[1]},$
%\begin{align}
%&\frac{\partial \phi_{i,k}^{[1]}(\bold{x}_i)}{\partial \bold{x}_i}=\frac{(\bold{x}_i-\bold{x}_k^{d})}{\|\bold{x}_i-\bold{x}_k^{d}\|}\neq\mathbf{0}_2, \forall \bold{x}_i\notin\mathcal T_{i,k}^{[1]},
%\end{align}
it follows from Theorem~11 in \cite{xu2015robustness} that $\bold{u}_i^{\ast}(\bold{x}_i)$ in \eqref{resilient_prioritization_optimization2} are locally Lipschitz continuous for  $\bold{x}_i\notin\mathcal T_{i,k}^{[1]}$. In this way, one has that $\bold{u}_i^{\ast}(\bold{x}_i)$ in \eqref{resilient_prioritization_optimization1} are locally Lipschitz continuous as well.
Using the asymptotic convergence in \eqref{convergence_property}, each robot~$i$ will finally converge to the target set $\mathcal T_{i,k}^{[1]}$, i.e., $\lim_{t\rightarrow\infty} \phi_{i, i}^{[1]}(\bold{x}_i(t))=0, \forall i\in\mathcal V$
because $\mathcal T_{i,k}^{[1]}$ only contains one desired point $x_k^d$, which further implies that 
\begin{align}
\label{conver_eq}
\lim_{t\rightarrow\infty} \Phi_{i}^{[1]}(t)=\mathbf{0}_N, \forall i\in\mathcal V.
\end{align}
Together with $\gamma(\cdot)$ defined in Definition~\ref{definition_minimum_energy}, it gives that 
\begin{align}
\label{conver_eq2}
\lim_{t\rightarrow\infty}\gamma((\Phi_{i}^{[1]}(t))=\mathbf{0}_N, \forall i\in\mathcal V.
\end{align}
Substituting Eqs.~\eqref{conver_eq} and \eqref{conver_eq2} into the constraints $\bold{g}_{i}^{[1]}\leq \mathbf{0}_N$ in \eqref{resilient_prioritization_optimization1} yields that the optimal value $\bm\delta_i^{\ast}$ can be minimized to be zeros as well, which implies that the optimal inputs of \eqref{resilient_prioritization_optimization1} will finally become $\bold{u}_i^{\ast}(\bold{x}_i)=\widehat{\bold{v}}_i$ to make the optimization always feasible. Finally, from \textbf{A6}, one has that $\lim_{t\rightarrow\infty}\bold{u}_i(t)=\bold{v}_d, \forall i\in\mathcal V$, which implies that each robot $i, i\in\mathcal V$ will maneuver with $\bold{v}_d$. The proof is thus completed.
\end{proof}

\subsubsection{Proof of {\bf P1}}
Since the path crossings $\mathcal C$ in {\bf P1} are generally defined to be the same positions of possibly different robots at possibly different historical time instances, which can only be calculated after a long period of time or once the formation maneuvering has been completed. Consequently, it is difficult to calculate and minimize the path crossings in real time for the robot group $\mathcal V$, which thus motivates the CATE algorithm~\eqref{resilient_prioritization_optimization} to indirectly and implicitly minimize possible path crossings by minimizing the desired-point allocation cost $b\|\bold{1}_N\t \bm\alpha -\bold{1}_N\t \|^2$ and the slack variables $\|\bm\delta_i\|^2$ simultaneously. Precisely, first, the allocation cost $\|\bold{1}_N\t \bm\alpha -\bold{1}_N\|=0$ and the integer constraints~(\ref{resilient_prioritization_optimization}d) ensure $N!$ choices of $\bm\alpha$ for an arbitrary one-to-one correspondence between desired points and robots. Then, the on-the-fly minimization of $\|\bm\delta_i\|^2$ further selects the most appropriate one-to-one correspondence among $N!$ choices between robots and their desired points when the obstacles occupy the passing space. In this way, the possible path crossings can be implicitly minimized.% by the flexible-ordering coordination in obstacle environments.%, which increases the efficiency of achieving {\it MPCM navigation}.

\begin{theorem}
\label{theo_CATE_obstacle}
Under Assumptions \textbf{A1}-\textbf{A6}, all robots $\mathcal V$ governed by \eqref{robot_dynamic}, the CATE algorithm~\eqref{resilient_prioritization_optimization} can achieve the {\it MPCM navigation in obstacle environments}, i.e., \textbf{P1-P4} in Problem~\ref{label_pro}.
\end{theorem}

\begin{proof}
We draw the conclusion from the analysis in Sections~\ref{sub_algorithm_feature} and \ref{sub_ana_P1_P4} directly.
\end{proof}

\begin{remark}
For the common undesirable deadlock problem where the trajectory of the robot is directed at some point toward the center of the obstacles \cite{grover2023before}, the proposed CATE algorithm \eqref{resilient_prioritization_optimization} can conveniently deal with it by reallocating the desired points to different robots to change the direction of the robot’s trajectory. Precisely, suppose the robot $i$, allocated to the desired point $k$, gets stuck and stops in a deadlock at some point due to the obstacles. It follows from  Lemma~\ref{lemma_final_convergence} that the remaining robots will keep moving toward their allocated desired points by reducing the slack vectors $\|\bm\delta_{i}\|^2$ in the cost functions \eqref{resilient_prioritization_optimization}. However, after some time, since robot $i$ remains stuck and the slack variable $\delta_{i,k}$ remains unchanged, the cost function in \eqref{resilient_prioritization_optimization} can no longer decrease. In such a situation, since the proposed CATE algorithm \eqref{resilient_prioritization_optimization} is updated on the fly, it will iteratively find a new allocation matrix $\bm\alpha$ to assign a different desired point $l$ to robot $i$ to further decrease the cost functions in \eqref{resilient_prioritization_optimization}. Therefore, the original direction of the robot $i$ will be changed, and it is no longer toward the center of obstacles, and thus robot $i$ breaks the original symmetry and escapes the undesirable deadlock situations, which implies that robot $i$ can continue the navigation task without interruption.
\end{remark}

\begin{table*}[]
\scriptsize
\centering
\begin{threeparttable}[b]
\caption{Summary of the problem's hardness in unlabeled multi-robot motion planning works.}
\label{table_hardness}
\belowrulesep=0pt
\aboverulesep=0pt
%\begin{tabular}{ | c | c | c | c | c | c | c | c |  c | c |}
\begin{tabular}{  >{\centering\arraybackslash}p{3.4cm}   >{\centering\arraybackslash}p{2.1cm}  >{\centering\arraybackslash}p{1.2cm} 
    >{\centering\arraybackslash}p{1.2cm} 
    >{\centering\arraybackslash}p{2.9cm} 
    >{\centering\arraybackslash}p{1.2cm} 
    >{\centering\arraybackslash}p{3.2cm} }
\toprule
Problem categrary&  Hardness \rule{0pt}{2ex} & Robot dynamics  & Obstacles & Optimization objective & Dimension & Positioning of robot, deseired points and obstacles\\
%\toprule
\midrule
%\rule{0pt}{2.5ex} Robots  & \diagbox{Methods}{Obstacles} &   4  &  5    &   6 &  7    &   4  &  5   &   6  &  7       \\
%\cmidrule(r){3-6}  \cmidrule(r){7-10} 
Unlabelled GMP\cite{solovey2016hardness,brocken2020multi} & PSPACE-hard \cite{hopcroft1984complexity} & \ding{55}\rule{0pt}{2.5ex}    &Static &  \ding{55}  & 2D  &  No requirement  \\
Unlabelled MPL-F\cite{yu2012distance,adler2015efficient,banyassady2022unlabeled} & P  & \ding{55}\rule{0pt}{2.5ex}  &  \ding{55}  & Shortest path length   & 2D    & Well separation   \\
Unlabelled MPL-O\cite{solovey2015motion} & P &  \ding{55}\rule{0pt}{2.5ex} & Static   &  Path-length minimization  & 2D    & Well separation  \\
Problem~\ref{label_pro} & Undetermined &  \ding{51}\rule{0pt}{2.5ex} &  Static/moving   &  Path-crossing minimization  &  2D/3D   & Assumptions {\bf A2}-{\bf A4} \\
 \midrule  
\end{tabular}
GMP: General motion planning, MPL-F: Minimum-path-length motion planning in free space, MPL-O: Minimum-path-length motion planning in obstacle environments.
   \end{threeparttable}
\end{table*}

 \begin{figure*}[!htb]
\centering
\includegraphics[width=16.2cm]{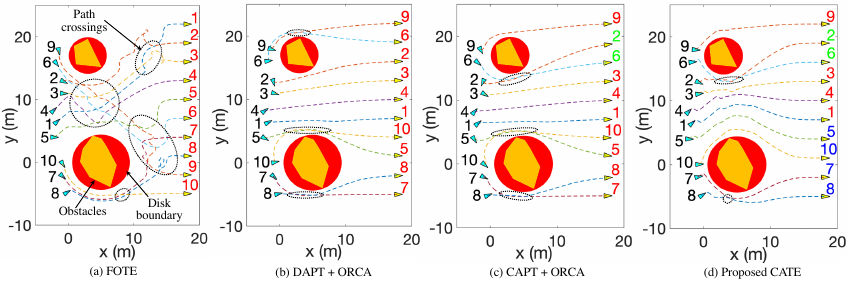}
\caption{ \textbf{First simulation: Column formation in obstacle environments.} 
(a) Trajectories of ten robots forming a column formation in the presence of two obstacles, using fixed-ordering task execution (FOTE)~\cite{notomista2021resilient}. (b) Trajectories of ten robots from the same initial setting to a column formation, using DAPT \cite{alonso2012image} + ORCA. (c) Trajectories of ten robots from the same initial setting to a column formation, using CAPT \cite{turpin2014capt} + ORCA. (d) Trajectories of ten robots from the same initial setting to a column formation, using the proposed CATE.  The blue and yellow triangles and the dashed lines denote the initial, final positions, and trajectories of the robots, respectively. The orange irregular polygons, red circles, and dashed ellipses are the obstacles, disk boundaries, and region of path crossings, respectively. Note that the green labels in subfigure (c) represent the switching of the desired-point assignments compared with the original red labels in subfigure (b). The blue labels in subfigure (d) denote the differences of the desired-point assignments compared with subfigure~(c).}
\label{sota_benchmark}
\end{figure*}

\subsection{Complexity Discussion of Problem~\ref{label_pro}}
%\begin{remark}
Due to the obstacles occupying the passing space between robots and their desired points, and the requirements for adapting to dynamic environments, it is not feasible to assign fixed costs for minimizing path-crossing points for each robot like the classic 1-1 matching problem  \cite{turpin2014capt,sabattini2017optimized}. To address this, we introduce the allocation cost $\|\bold{1}_N\t \bm\alpha -\bold{1}_N\t \|$ together with the constraint in (\ref{resilient_prioritization_optimization}e) that simply restricts the solution space to the $N!$ possible one-to-one mappings, rather than computing the optimal assignment directly.  Subsequently, the additional costs of slack variable $\|\bm\delta_i\|^2$ further select the most appropriate assignment. Although the cost $\|\bold{1}_N\t \bm\alpha -\bold{1}_N\t \|$ is influenced by the allocation of other robots, the common objective of the one-to-one correspondence in \eqref{eq_con_a_rob_a_poi} and \eqref{eq_con_a_poi_a_rob} remains unaffected. Moreover, this cost also offers flexibility and adaptability in matching, particularly in obstacles and dynamic environments. 
However, in contrast to~\cite{turpin2014capt,sabattini2017optimized}, this cost design makes it difficult to determine whether the matching problem is efficiently solvable. Currently, the CATE algorithm \eqref{resilient_prioritization_optimization} is formulated to be the MIQP framework that is solved using the Gurobi optimization tool \cite{gurobi}. Furthermore, as shown by the shortest convergence time of our algorithm in Table~\ref{table_EM}, the matching problem is solved with favorable performance. The theoretical analysis of whether the matching in our problem is optimal or not will be a future direction.
%\end{remark}

\begin{table*}[]
\scriptsize
\centering
\begin{threeparttable}[b]
\caption{\textbf{Second simulation: Statistical comparison of the ``Arrow formation" using different methods} in 16 different robot-obstacle group settings.}
\label{table_EM}
\belowrulesep=0pt
\aboverulesep=0pt
%\begin{tabular}{ | c | c | c | c | c | c | c | c |  c | c |}
\begin{tabular}{  >{\centering\arraybackslash}p{1.6cm}   >{\centering\arraybackslash}p{2.35cm}  >{\centering\arraybackslash}p{1.2cm} 
    >{\centering\arraybackslash}p{1.2cm} 
    >{\centering\arraybackslash}p{1.2cm} 
    >{\centering\arraybackslash}p{1.2cm} 
    >{\centering\arraybackslash}p{1.2cm} 
    >{\centering\arraybackslash}p{1.2cm} 
    >{\centering\arraybackslash}p{1.2cm}
     >{\centering\arraybackslash}p{1.2cm} }
\toprule
 &  Metrics \rule{0pt}{2ex} & \multicolumn{4}{c}{Success rates (\%)} & \multicolumn{4}{c}{Convergence time (s):  Mean (Standard Deviation)}   \\
%\toprule
\midrule
\rule{0pt}{2.5ex} Robots  & \diagbox{Methods}{Obstacles} &   4  &  5    &   6 &  7    &   4  &  5   &   6  &  7       \\
\cmidrule(r){3-6}  \cmidrule(r){7-10} 
 \multirow{4}{*}{5} & FOTE & 100\rule{0pt}{2.5ex}    &60   &  100  & 90  &  13.70(3.1) &  12.43(2.3)    & 13.48(3.0) &  12.62(2.7)  \\
 & DAPT+ORCA  & 90\rule{0pt}{2.5ex}  & 100   & 100  & 90    & 22.53(3.8)  & 23.88(2.8)  & 20.12(3.3) & 22.03(3.5)   \\
 & CAPT+ORCA &  80\rule{0pt}{2.5ex} & 100   &  100  & 100    & 25.18(0.5)  & 25.28(0.6)  & 25.34(0.2) & 24.32(0.7)    \\
 & Ours(CATE) &  \textbf{100}\rule{0pt}{2.5ex} &  \textbf{100}   &  \textbf{100}  &  \textbf{100}   & \textbf{9.44(1.1)} &    \textbf{10.19(2.1)}   &  \textbf{10.08(1.1)}  &    \textbf{10.90(2.3)}  \\
 \midrule  
   
 \multirow{4}{*}{7} 
 & FOTE & 100\rule{0pt}{2.5ex}    & 100    & 90    &  70  & 11.10(2.2)   &  14.12(2.8)    & 13.94(2.9)  & 14.86(3.1)   \\
 & DAPT+ORCA  & 100\rule{0pt}{2.5ex}  &  100   & 80   & 70    &    23.66(2.7)    &  24.29(2.9)    & 23.29(2.9)    & 22.63(2.8)    \\
 & CAPT+ORCA &  90\rule{0pt}{2.5ex} &  100     & 80   & 70   & 25.27(0.2)    &  25.24(0.3)     &  25.27(0.5)   &  25.22(0.2)  \\
 & Ours(CATE) &  \textbf{100}\rule{0pt}{2.5ex} &   \textbf{100}      &  \textbf{100}  &  \textbf{100}    &    \textbf{9.41(0.5)}    &  \textbf{9.81(0.6)}   &  \textbf{11.44(3.0)}  &    \textbf{9.85(0.6)}  \\
 \midrule
 
 \multirow{4}{*}{9} 
 & FOTE & 100   &  100   &  100  & 100   &  15.01(2.4) & 13.16(3.1) &  14.27(2.5)  & 14.80(2.5)\rule{0pt}{2.5ex}  \\
 & DAPT+ORCA  & 70\rule{0pt}{2.5ex}  &  70  &  60  &  60  & 25.28(0.2) &  24.66(2.1) & 25.36(0.1) & 23.73(3.2) \\
 & CAPT+ORCA &  70\rule{0pt}{2.5ex} & 70    &  40   &  70   & 25.11(0.1) &  25.14(0.2)  & 25.09(0.2) & 25.17(0.2)  \\
 & Ours(CATE) &  \textbf{100}\rule{0pt}{2.5ex} &   \textbf{100}    &  \textbf{100}   &   \textbf{100}   &  \textbf{10.36(2.1)}  &  \textbf{9.60(0.9)}  &    \textbf{9.34(0.2)} &   \textbf{9.43}(0.5) \\
 \midrule 
 
 \multirow{4}{*}{11} 
 & FOTE &100    &  100    & 100   & 100   &  14.00(2.5)  & 14.94(2.7)  &  13.28(2.9) & 15.34(2.5)\rule{0pt}{2.5ex} \\
 & DAPT+ORCA  &60\rule{0pt}{2.5ex} &  90    & 80   & 30    & 25.28(0.5) &  24.65(2.1) & 25.3(0.08) & 18.96(8.3) \\
 & CAPT+ORCA & 100\rule{0pt}{2.5ex}  &  90   &  60  & 50    & 25.14(0.2) &  25.24(0.3) & 25.12(0.3) & 25.39(0.5) \\
 & Ours(CATE) &  \textbf{100}   &   \textbf{100}   &    \textbf{100}  & \textbf{100}    &  \textbf{9.91(1.01)}  &  \textbf{9.75(0.55)} &   \textbf{10.63(2.5)} &  \textbf{11.00(3.1)}\rule{0pt}{2.5ex}\\
 \midrule 
\end{tabular}
For each method, we conduct $16\times10=160$ trials of different robot-obstacle groups. 
1) {\it Success rate}: the ratio of the successful {\it MPCM} navigation in 10 trials of each robot-obstacle group. 
2) {\it Convergence time} $T_{r}:=\inf\{t~\big|~\| \bold{x}_i(t)-\bold{x}_{s[i]}^d(t)\|\leq\varrho, \forall i\in\mathcal V\}$ with  $\varrho=0.2$, which is the time when the distance between each robot and its desired point is small than~$\varrho$.
%\begin{tablenotes}
% \item[1] N/A: Not applicable.
%\end{tablenotes}
   \end{threeparttable}
\end{table*}

%\begin{remark}
According to the property {\bf P2} in Definition~\ref{definition_MPCM}, Problem~\ref{label_pro} falls within the category of unlabeled multi-robot motion planning. Based on the summary and comparison of the problem hardness in Table~\ref{table_hardness}, it is challenging to determine the hardness of Problem~\ref{label_pro}. \textbf{(i)} On one hand, if we ignore the robot dynamics, the optimization objective of path-crossing minimization, the well-separation condition of robots, desired points, and obstacles, and only consider static obstacles and 2D environments, there exists a polynomial-time reduction from
the special case in \cite{solovey2016hardness} to the simplified version of Problem~\ref{label_pro}, proving that the simplified Problem~\ref{label_pro} is also P-SPACE hard.
%
%then Problem~\ref{label_pro} becomes the special case in \cite{solovey2016hardness}, where the hardness is at least PSPACE-hard. 
Unfortunately, there is no existing work that proves the theoretical hardness of unlabeled multi-robot motion planning with robot dynamics, although practical solutions have been proposed using efficient sample-based algorithms \cite{le2021multi}. If any of the ignored elements are reintroduced, the hardness of Problem~\ref{label_pro} may become even harder than the conjectured P-SPACE hardness. \textbf{(ii)} On the other hand, if the radius of the robot is selected to be 4/r, Assumptions {\bf A2-A4} in Problem~\ref{label_pro} are similar to the well-separation condition in \cite{solovey2015motion}, with the slight distinction in the initial relative positions between robots and desired points. This similarity may reduce the hardness of Problem~\ref{label_pro} to be solvable in polynomial time. However, as shown in Table~\ref{table_hardness}, Problem~\ref{label_pro} still involves additional constraints of the robot dynamics, moving obstacles, different optimization objectives, and 3D environments. These factors make it challenging to find a provable and optimal solution that runs in polynomial time. Although we provide an MIQP solution (i.e., the CATE algorithm \eqref{resilient_prioritization_optimization}), this approach cannot be used to determine the hardness of Problem~\ref{label_pro} in theory.
\textbf{(iii)} Finally, due to the absence of a suitable reduction framework, determining the hardness of Problem~\ref{label_pro} under the aforementioned additional constraints is a non-trivial task. This challenge 
represents a promising direction for future research in its own right.
%\end{remark}

\begin{figure}[!htb]
\centering
\includegraphics[width=6.5cm]{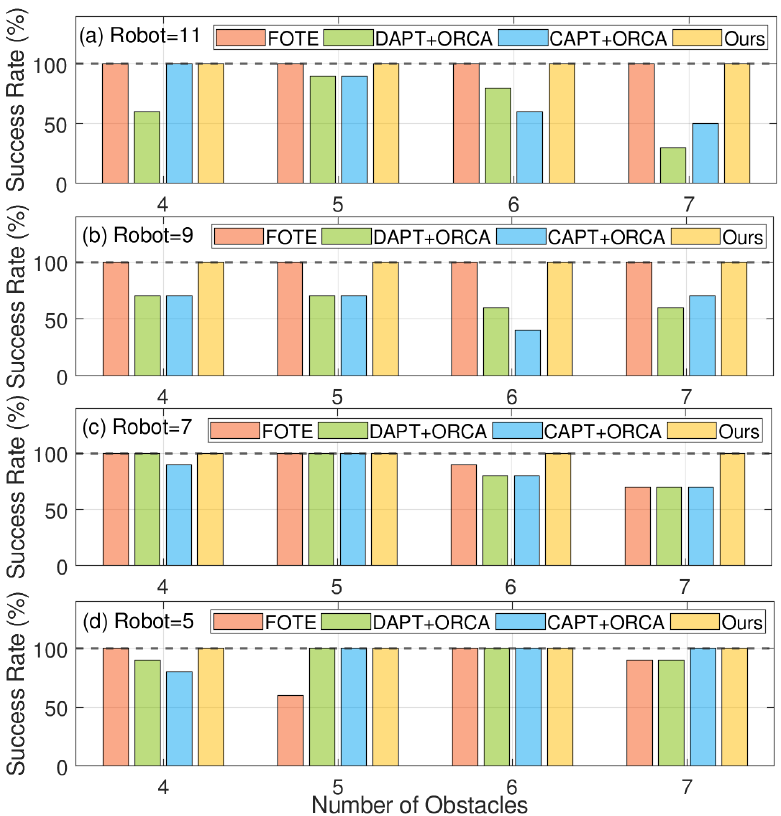}
\caption{Bar figure comparison of the success rates for three baseline methods and the proposed CATE algorithm \eqref{resilient_prioritization_optimization} in 16 different robot-obstacle groups.} 
\label{success_rate}
\end{figure}

\section{Simulations and Experiments}
\label{sec_verification}
In this section, we perform simulations to validate the effectiveness, adaptability and feasibility of the CATE algorithm \eqref{resilient_prioritization_optimization} in obstacle environments, and then conduct AMRs experiments to demonstrate its efficacy.

%perform simulations to validate the effectiveness and robustness of the CATE algorithm \eqref{resilient_prioritization_optimization}, and conduct AMRs experiments to demonstrate its efficacy. 
%The results can also be viewed in the supplementary video\footnote{[Online]. Available: \href{https://www.youtube.com/watch?v=76kJvtpTToM}{https://www.youtube.com/watch?v=76kJvtpTToM}}.

\begin{figure}[!htb]
\centering
\includegraphics[width=6.55cm]{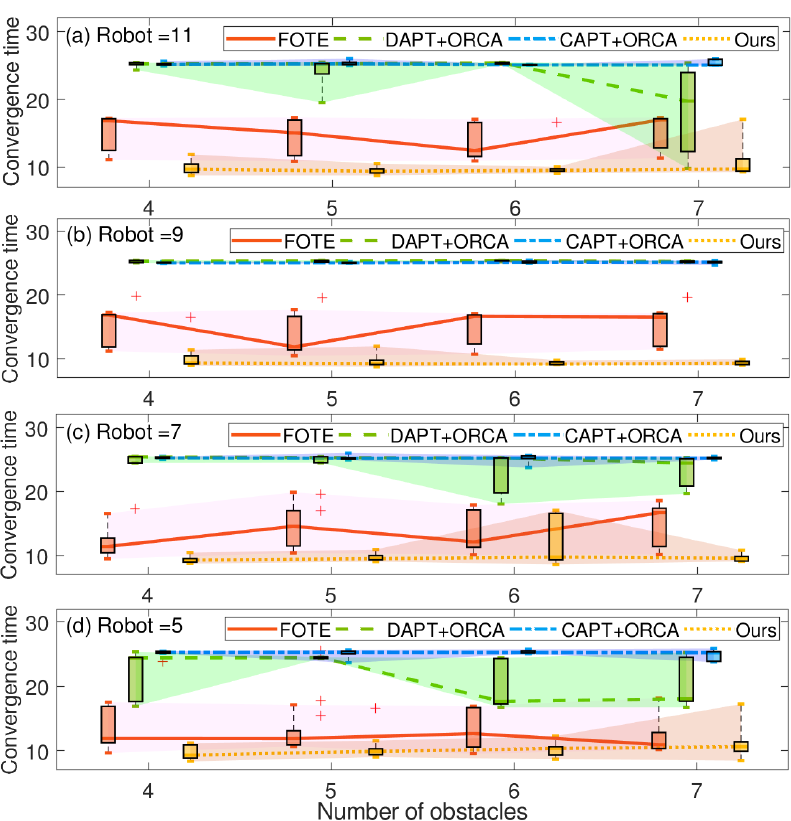}
\caption{Boxplot comparison of the convergence times for three baseline methods and the proposed CATE algorithm \eqref{resilient_prioritization_optimization} in 16 different robot-obstacle groups. }
\label{convergence_time}
\end{figure}

\subsection{Accommodating the CATE to Unicycle Dynamics}
Recalling the single-integrator dynamics ~\eqref{robot_dynamic}, the input $u_i$ can be regarded as high-level guidance commands for arbitrary robots of specific dynamics as long as the low-level velocity tracking is regulated in time. Therefore, we first accommodate the CATE signals~\eqref{resilient_prioritization_optimization} to a practical unicycle model, which is described by the common input-output-state feedback linearization to control mobile robots \cite{khalil2002nonlinear,kim1999tracking,tnunay2017distributed},
\begin{align}
\label{unicycle}
\dot{x}_{i,1}=&v_i\cos\theta_i,~\dot{x}_{i,2}=v_i\sin\theta_i,~\dot{\theta}_i=u_{\theta_i},\nonumber\\
(\dot{x}_{i,3}=&u_{i,z}~\mbox{if robots are in 3D}), i\in\mathcal V,
\end{align}
where $\bold{x}_i=[x_{i,1}, x_{i,2}, (x_{i,3})]\t\in\mathbb{R}^{3}$ is the position, $\theta_i\in \mathbb{R}$ is the heading angle in X-Y plane, $v_i\in\mathbb{R}, u_{\theta_i}\in\mathbb{R}, u_{i,z}\in\mathbb{R}$ are the linear velocity, angle velocity, and climbing velocity, respectively. Then, it follows from the near-identity diffeomorphism \cite{glotfelter2019hybrid} that $\bold{u}_i^{\ast}$ in CATE algorithm \eqref{resilient_prioritization_optimization} can be transformed to the desired velocities $\{v_i, u_{\theta_i}, u_{i,3}\}$ in~\eqref{unicycle} as shown below
\begin{align}
\label{unicycle_transformation}
&v_{i}=u_{i,1}^{\ast}\cos\theta_i +u_{i,2}^{\ast}\sin\theta_i, u_{\theta_i}=-\frac{u_{i,1}^{\ast}}{l}\sin\theta_i +\frac{u_{i,2}^{\ast}}{l} \cos\theta_i, \nonumber\\
&(u_{i,z}=u_{i,3}^{\ast}~\mbox{if robots are in 3D}), i\in\mathcal V,
\end{align}
with $\bold{u}_{i}^{\ast}=[u_{i, 1}^{\ast}, u_{i, 2}^{\ast}, u_{i, 3}^{\ast}]\t$ and $l\in\mathbb{R}^+$ being an arbitrary small constant.

\begin{table*}[]
\scriptsize
\centering
\begin{threeparttable}[b]
\caption{\textbf{Second simulation:} Quantitative evaluation metrics of the multi-robot navigation efficiency with 16 different robot-obstacle group settings.}
\label{table_EM2}
\belowrulesep=0pt
\aboverulesep=0pt
%\begin{tabular}{ | c | c | c | c | c | c | c | c |  c | c |}
\begin{tabular}{  >{\centering\arraybackslash}p{1.6cm}   >{\centering\arraybackslash}p{2.35cm}  >{\centering\arraybackslash}p{1.2cm} 
    >{\centering\arraybackslash}p{1.2cm} 
    >{\centering\arraybackslash}p{1.2cm} 
    >{\centering\arraybackslash}p{1.2cm} 
    >{\centering\arraybackslash}p{1.2cm} 
    >{\centering\arraybackslash}p{1.2cm} 
    >{\centering\arraybackslash}p{1.2cm}
     >{\centering\arraybackslash}p{1.2cm} }
\toprule
 &  Metrics \rule{0pt}{2ex} & \multicolumn{4}{c}{Path crossings:  Mean (Standard Deviation)} & \multicolumn{4}{c}{Trajectory length (m):  Mean (Standard Deviation)}   \\
%\toprule
\midrule
\rule{0pt}{2.5ex} Robots  & \diagbox{Methods}{Obstacles} &   4 & 5    &   6  &  7    &   4  &  5   &   6  &  7       \\
\cmidrule(r){3-4} \cmidrule(r){5-6} \cmidrule(r){7-8} \cmidrule(r){9-10}
 \multirow{4}{*}{5} & FOTE & 9.2(2.9)\rule{0pt}{2.5ex}    &6.5(2.5)  &  6.2(3.92) & 5.89(2.62) & 129.4(21.2)   & 92.68(5.6)   &  89.07(15.5)     &  91.28(13.8)  \\
 & DAPT+ORCA  & 2.67(1.2)\rule{0pt}{2.5ex}  &2.9(1.70)   & 2.5(1.50)  & 2.78(1.03)   &  87.5(8.2)  & 87.21(5.1)    & 87.09(5.6)  & 85.46(3.4)   \\
 & CAPT+ORCA &  \textbf{ 2.13(0.9)}\rule{0pt}{2.5ex} & \textbf{1.5(0.95)} & \textbf{1.7(0.90)}  & \textbf{ 2.4(1.69)}  &   85.01(4.5) & 84.46(6.0)      & 84.83(4.6)  & 83.78(3.1)    \\
 & Ours(CATE) &  2.9(1.8)\rule{0pt}{2.5ex} &  2.3(2.00)  & 3.6(1.62) & 2.8(1.83) &   \textbf{76.15}(11.4) &   \textbf{76.90}(7.2)    &   \textbf{78.92}(8.2)     &   \textbf{81.95}(13.2)  \\
 \midrule  
   
 \multirow{4}{*}{7} 
 & FOTE & 14.3(8.9)\rule{0pt}{2.5ex}    & 14(4.47) & 12.78(3.6) & 12.43(3.8)  &  127.0(16.1)  &  142.0(23.5)    &  144.7(20.7)    & 137.0(10.3)   \\
 & DAPT+ORCA  & 6.4(2.9)\rule{0pt}{2.5ex}  &   5.5(1.50)   & 4.86(1.6) & 6.67(6.0) &  119.77(3.2)  & 120.52(5.6)   &  116.92(5.5)    & 122.63(3.1)    \\
 & CAPT+ORCA &  \textbf{ 4(2)}\rule{0pt}{2.5ex} &  \textbf{ 3.5(1.50)}  & \textbf{ 4.63(1.3)} & \textbf{ 4.57(2.1)}    &  116.57(4.1)  & 118.11(5.2)    &  118.15(5.1)     &  114.20(4.4)  \\
 & Ours(CATE) &  7.7(2.1)\rule{0pt}{2.5ex} &   6.2(3.0)   & 8.4(4.5)   & 6.4(1.7) &   \textbf{101.48}(8.6) &  \textbf{108.1}(10.1)     &  \textbf{108.7}(12.1)    &   \textbf{106.94}(4.2)    \\
 \midrule
 
 \multirow{4}{*}{9} 
 & FOTE & 26.5(11.1)   &  30.7(10.8)  & 28.5(14.5) & 26.7(5.6) &  182.8(17.7)   & 183.9(25.7)     & 189.1(22.9)    & 194.0(19.0)\rule{0pt}{2.5ex}  \\
 & DAPT+ORCA  & 9.86(3.6)\rule{0pt}{2.5ex}  &  12.14(3.3)  & 14.67(1.9) & 12.33(7.1) &   160.74(5.6)   &  164.4(6.9)   &  164.4(5.8)    & 164.85(5.2)   \\
 & CAPT+ORCA &  9.43(2.4)\rule{0pt}{2.5ex} & 8.71(3.20)   & 8.25(3.8) & \textbf{7.57(2.5)}  &  153.92(7.6)    &  153.46(5.5)    & 151.81(7.7)    & 156.64(5.1)   \\
 & Ours(CATE) & \textbf{8.4(3.3)}\rule{0pt}{2.5ex} &  \textbf{8.7(2.00) }  &\textbf{ 7.2(3.2)} &   8.4(3.14) &   \textbf{136.6}(12.5) &  \textbf{130.6}(11.3)     &  \textbf{132.28}(6.6)  & \textbf{139.32}(6.8)  \\
 \midrule 
 
 \multirow{4}{*}{11} 
 & FOTE &36.4(10.8)    &  40.1(8.2)   & 39.7(6.26)  & 38(14.59)  &  232.9(13.9)  & 231.6(17.1)   & 232.3(25.2)    &  251.69(8.7)\rule{0pt}{2.5ex} \\
 & DAPT+ORCA  &15.5(4.5)\rule{0pt}{2.5ex} & 15.67(4.4)  & 11.25(3.6)  & 12(0.86) & 201.50(6.3)  & 201.64(6.3)      & 199.82(3.8)   & 197.03(7.0)  \\
 & CAPT+ORCA & \textbf{ 9.1(2.81)}\rule{0pt}{2.5ex}  &  \textbf{ 10.89(2.9)} &   \textbf{11(3.65)}  & \textbf{10(2.19)} & 181.96(7.3)  & 189.2(10.3)     &  186.03(3.2)   & 186.37(5.2)    \\
 & Ours(CATE) &  10.7(3.4)   &   11.1(3.0)  & 11.2(2.89)  & 13.5(7.95) &   \textbf{164.89}(9.8) & \textbf{162.58}(7.3)     & \textbf{168.4}(10.5)  & \textbf{176.0}(14.6)\rule{0pt}{2.5ex}\\
 \midrule 
\end{tabular}
1) {\it Path crossings} are the number of crossing points among different paths during the convergence time $T_r$. 2) {\it Trajectory length} is the total length of all robots' paths during the convergence time $T_r$. Note that the path crossings and trajectory lengths are only calculated if the {\it MPCM} navigation is successfully achieved.
%\begin{tablenotes}
%     \item[1] {\blue Note that the path crossings and trajectory lengths are only calculated if the {\it MPCM} navigation is successfully achieved. }
%\end{tablenotes}
   \end{threeparttable}
\end{table*}
\begin{figure}[!htb]
\centering
\includegraphics[width=7.2cm]{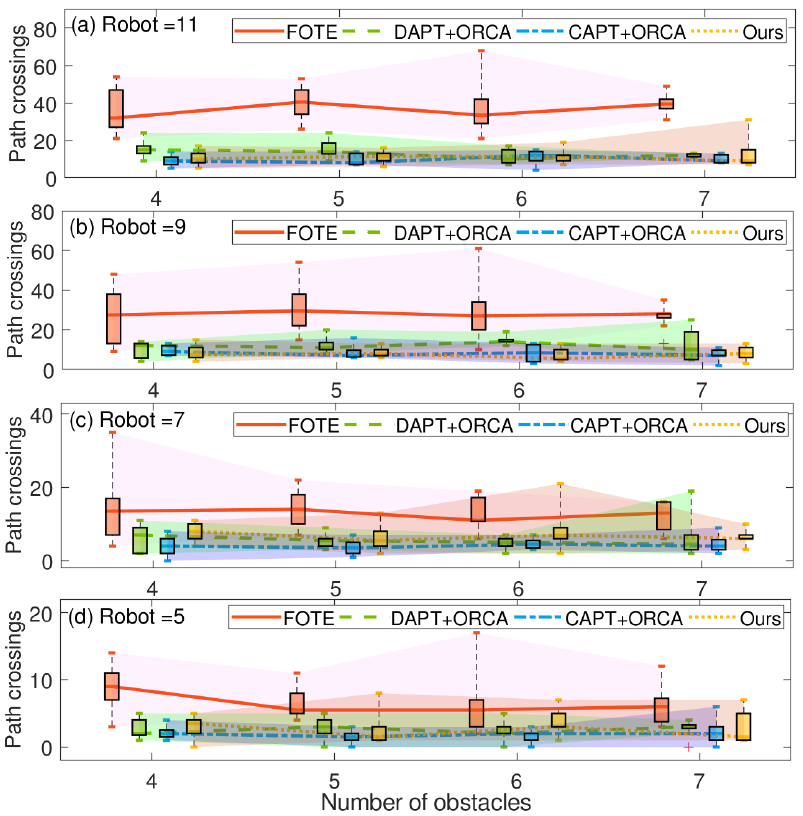}
\caption{Boxplot comparison of the path crossings for three baseline methods and the proposed CATE algorithm \eqref{resilient_prioritization_optimization} in 16 different robot-obstacle groups. } 
\label{path_crossing}
\end{figure}

\begin{figure}[!htb]
\centering
\includegraphics[width=7.25cm]{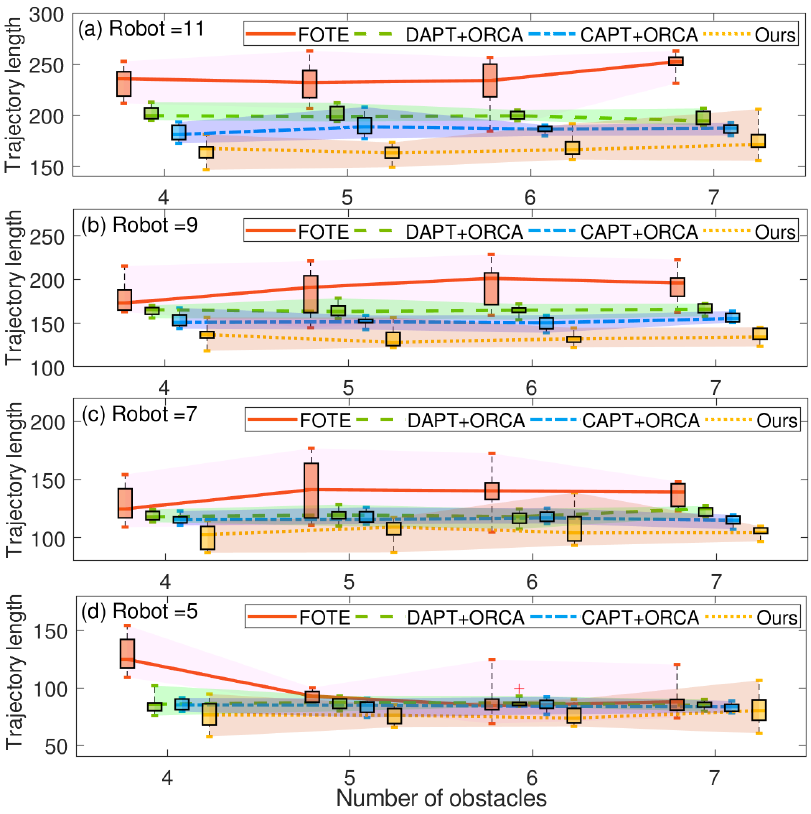}
\caption{Boxplot comparison of the trajectory length for three baseline methods and the proposed CATE algorithm \eqref{resilient_prioritization_optimization} in 16 different robot-obstacle groups. } 
\label{trajectory_length}
\end{figure}

 \begin{figure*}[!htb]
\centering
\includegraphics[width=16cm]{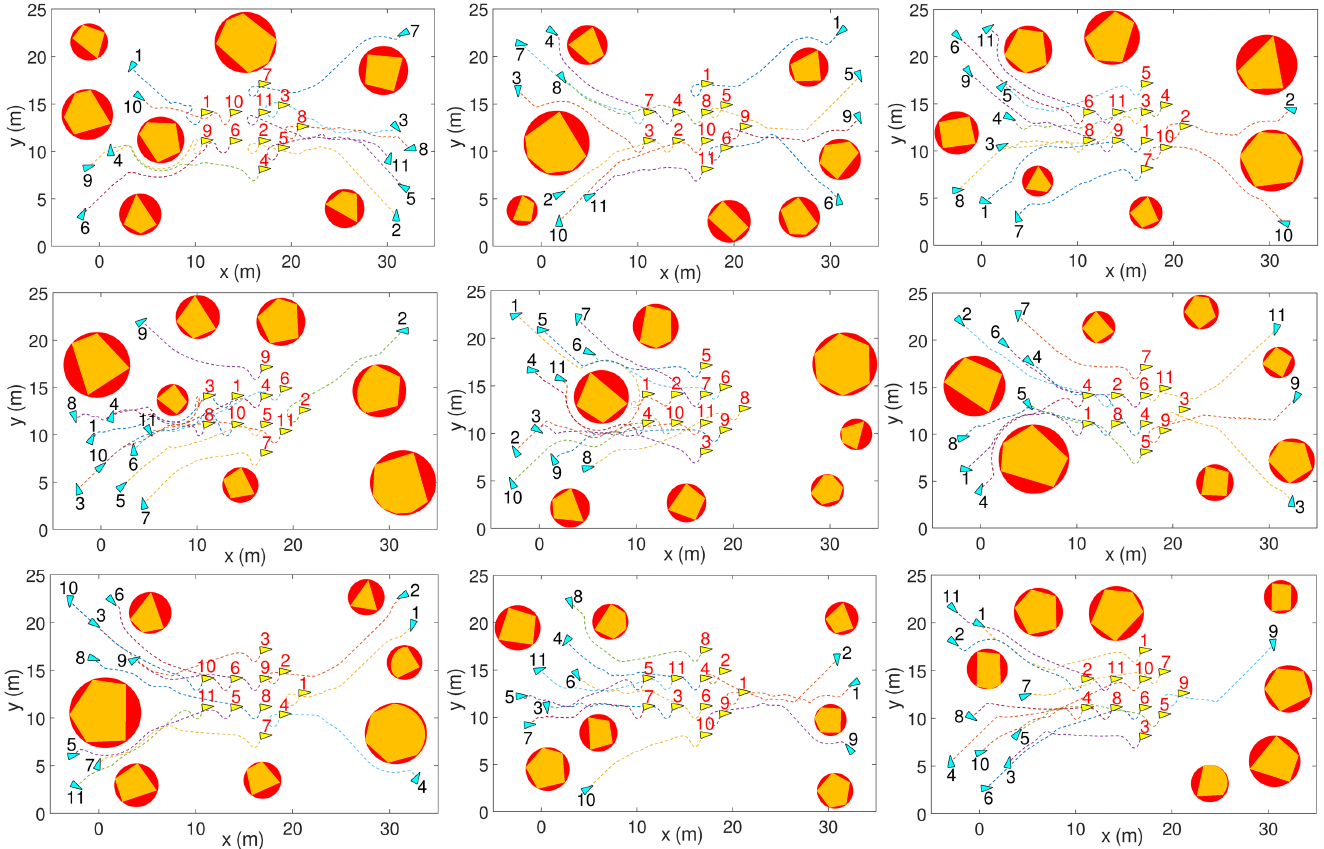}
\caption{Nine illustrative examples of the ``Arrow" formation using the proposed CATE algorithm~\eqref{resilient_prioritization_optimization} in the \textbf{second simulation}.
In the detailed nine trials, the numbers of robots and obstacles are set to be the same, i.e., $N=11$ and $M=7$, and the size of the obstacle circle is randomly selected in the region of $r\in[1.7, 4]$. The positions of the robots and obstacles are randomly selected within the space of $\{x\in[-5,~35], y\in[0,~25]\}$, but do not overlap each other or in the space of $\{x\in[10,~25], y\in[7,~18]\}$.
 The abbreviations, blue and yellow triangles, red circles, and dashed colored lines have the same meanings as those in Fig.~\ref{sota_benchmark}.}
\label{example_3}
\end{figure*}

\subsection{Baseline Methods}
\label{base_line}
To comparatively evaluate the performance of the proposed CATE algorithm~\eqref{resilient_prioritization_optimization}, we use the following state-of-the-art baselines to conduct the same simulations and experiments.

\subsubsection{FOTE~\cite{notomista2021resilient}} a state-of-the-art general navigation method that can be used to form the navigation in the obstacle environments with fixed-ordering sequences.

\subsubsection{DAPT\cite{alonso2012image} + ORCA~\cite{van2008reciprocal}} the DAPT algorithm \cite{alonso2012image} is a decoupled method (see Fig.~\ref{previous_PCM}) to generate non-intersecting paths in obstacle-free {\it MPCM navigation} mission, which cannot be adapted to compare with the proposed CATE algorithm~\eqref{resilient_prioritization_optimization} in obstacle environments directly. However, we add the ORCA method \cite{van2008reciprocal} by regarding obstacles as static robots to avoid collision avoidance. 

\subsubsection{CAPT\cite{turpin2014capt} + ORCA~\cite{van2008reciprocal}} the CAPT algorithm \cite{turpin2014capt} is a concurrent allocation and planning method (see Fig.~\ref{previous_PCM}) which can generate non-intersecting paths more efficiently in obstacle-free {\it MPCM navigation} mission. However,  such kind of method cannot be adapted to compare with the proposed CATE algorithm \eqref{resilient_prioritization_optimization} in obstacle environments as well. Therefore, we also add the ORCA method \cite{van2008reciprocal}  by regarding obstacles as static robots to avoid collision avoidance.

% \begin{figure*}[!htb]
%\centering
%\includegraphics[width=\hsize]{comparision.pdf}
%\caption{ \textbf{Third simulation: Adaptability of the CATE algorithm \eqref{resilient_prioritization_optimization} in 3D.} The blue and yellow triangles, black and red labels have the same meanings as those in Fig.~\ref{sota_benchmark}. The dark blue balls denotes the 3D obstacles. }
%\label{comparision}
%\end{figure*}

\subsection{Simulations}
\label{algo_simulations}
For all simulations, we choose the same parameters for convenience. Precisely, $l$ in \eqref{unicycle_transformation} is set to be $l=0.5$, the input limit of each robot is set to be $u_{\max}=3$, i.e., $\|\bold{u}_i\|\leq3$. Using \textbf{A5}, the sensing and collision radii are set to be $R=4$ and $r=1$, respectively. The weights $b$ and $c$ in~(\ref{resilient_prioritization_optimization}a) are set to be $b=10^5, c=10^2$, respectively, and the penalty parameter $\varpi$ in~(\ref{resilient_prioritization_optimization}b) is set to be $\varpi=1000$. As for the distributed estimator for $v_d$ in~\eqref{resilient_prioritization_optimization}, one can design $\widehat{\bold{v}}_i, i\in\mathcal V,$ according to~\cite{hong2006tracking} with a cyclic connected topology, which fulfills \textbf{A6}. %ssumption~\ref{assum_velocity_estimate}. %i.e., Objective \ref{obj_maneuvering}) of Definition~\ref{definition_navigation} is achieved in advance.

In the first simulation (see Fig.~\ref{sota_benchmark}), to show the effectiveness of the CATE algorithm \eqref{resilient_prioritization_optimization}, we compare it with another three baseline methods in Section~\ref{base_line}. Precisely, we consider ten robots achieving the column formation with two circular obstacles, which are located at $\bold{x}_1^{o}=[5,0]\t, \bold{x}_2^o=[3,17]\t$ satisfying $\phi_{i,l}^{[3]}(\bm\sigma_i, \bold{x}_l^o):=1+r_l^o-\|\bm\sigma_i-\bold{x}_l^o\|\leq0, l=1,2$ and $r_1^o=4.5, r_2^o=3$. The desired points are set according to \textbf{A2}. Figs.~\ref{sota_benchmark}~(a)-(d) demonstrate the temporal evolution of robots from the same initial states (blue triangles) fulfilling~\textbf{A3, A4} to the successful column formation (yellow triangles) by using four different methods. It is observed in Figs.~\ref{sota_benchmark} (b)-(d) that the latter three methods can reduce path crossings (i.e., dashed-ellipse regions) by the flexible-ordering column coordination.
However, Fig.~\ref{sota_benchmark} (d) exists the least path crossings (i.e., dashed-ellipse regions) than Figs.~\ref{sota_benchmark} (a)-(c), which implies that the proposed CATE algorithm~\eqref{resilient_prioritization_optimization} endows higher efficiency than the other three methods in the obstacle environments.
%because the flexible allocation and obstacle constraints in \eqref{resilient_prioritization_optimization} also account for obstacles rather than directly relying on the minimal squared-distance assignment among robots.

In the second simulations (see Tables~\ref{table_EM}-\ref{table_EM2} and Figs.~\ref{success_rate}-\ref{failure}), to better compare the performance with alternative approaches in Section~\ref{base_line}, we conduct extensive experiments of different robot-obstacle combinations and analyze the aggregated results. Precisely, as shown in Table~\ref{table_EM}, we select 16 groups of different numbers of robots and obstacles (i.e., robots: $N=5, 7, 9, 11$, obstacles $M=4,5,6,7$), and then conduct 10 trials for each group. In each trial, the size of the obstacle circle is randomly selected in the range of $r\in[1.7, 4]$. The positions of the robots and obstacles are randomly chosen within the area of $\{x\in[-5,~35], y\in[0,~25]\}$, but do not overlap other robots, obstacles, or in the center space of $\{x\in[10,~25], y\in[7,~18]\}$.

\begin{figure*}[!htb]
\centering
\includegraphics[width=16cm]{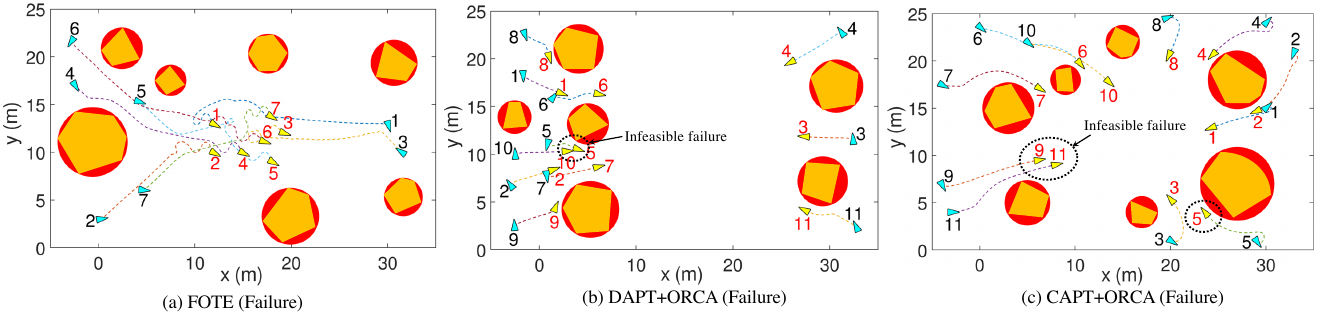}
\caption{ Three failure examples using the previous three different baseline methods in the \textbf{second simulation}. (a) The deadlock failure in the $7$-robot-$7$-obstacle scenario when using the FOTE method. (b) The infeasible failure in the $9$-robot-$6$-obstacle scenario when using the DAPT+ORCA method.  (c) The infeasible failure in the $11$-robot-$7$-obstacle scenario when using the CAPT+ORCA method.}
\label{failure}
\end{figure*}

\begin{figure*}[!htb]
\centering
\includegraphics[width=16cm]{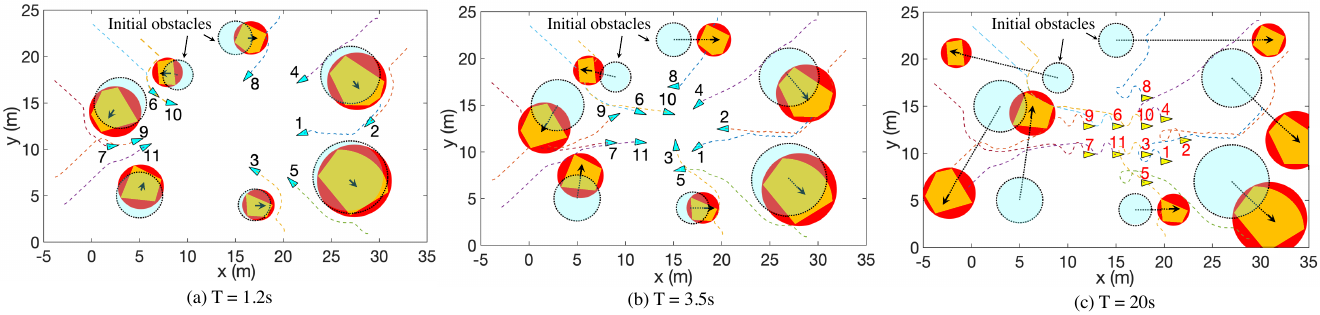}
\caption{\textbf{Third simulation: Effectiveness of the CATE algorithm \eqref{resilient_prioritization_optimization} in environments with moving obstacles.} (a) T=$1.2$s. (b) T=$3.5$s. (c) T=$20$s. The light blue circles represent the initial positions of the obstacles. The abbreviation, blue and yellow triangles, red circles, and dashed colored lines have the same meanings as those in Fig.~\ref{sota_benchmark}.}
\label{moving_obstacle}
\end{figure*}

 \begin{figure}[!htb]
\centering
\includegraphics[width=8cm]{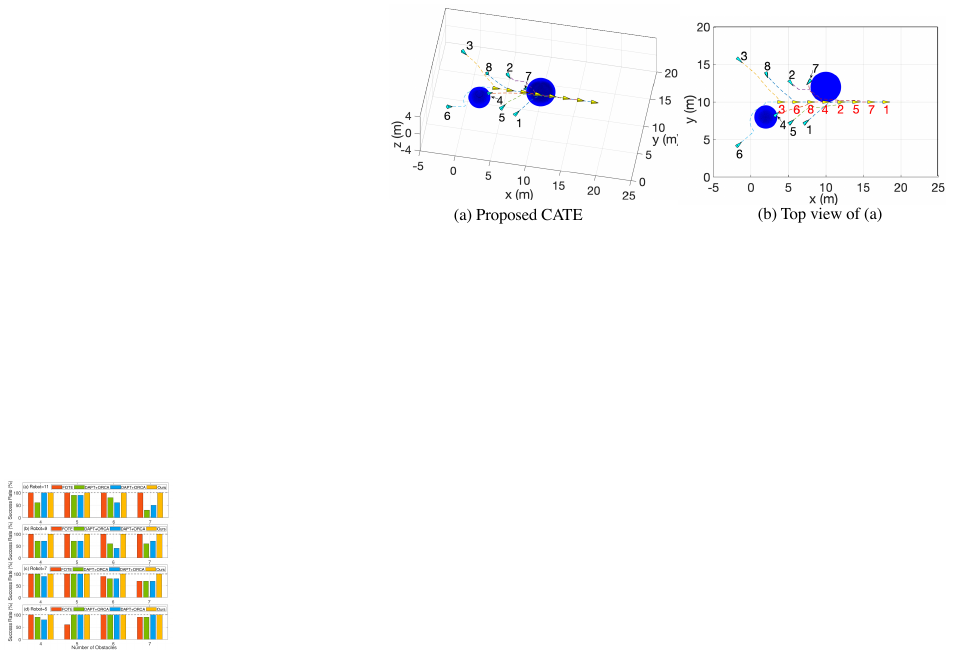}
\caption{\textbf{Fourth simulation: Adaptability of the CATE algorithm \eqref{resilient_prioritization_optimization} in 3D.} The blue and yellow triangles, and black and red labels have the same meanings as those in Fig.~\ref{sota_benchmark}. The dark blue balls denote the 3D obstacles. }
\label{rec_obstacle}
\end{figure}

 \begin{figure}[!htb]
\centering
\includegraphics[width=7.0cm]{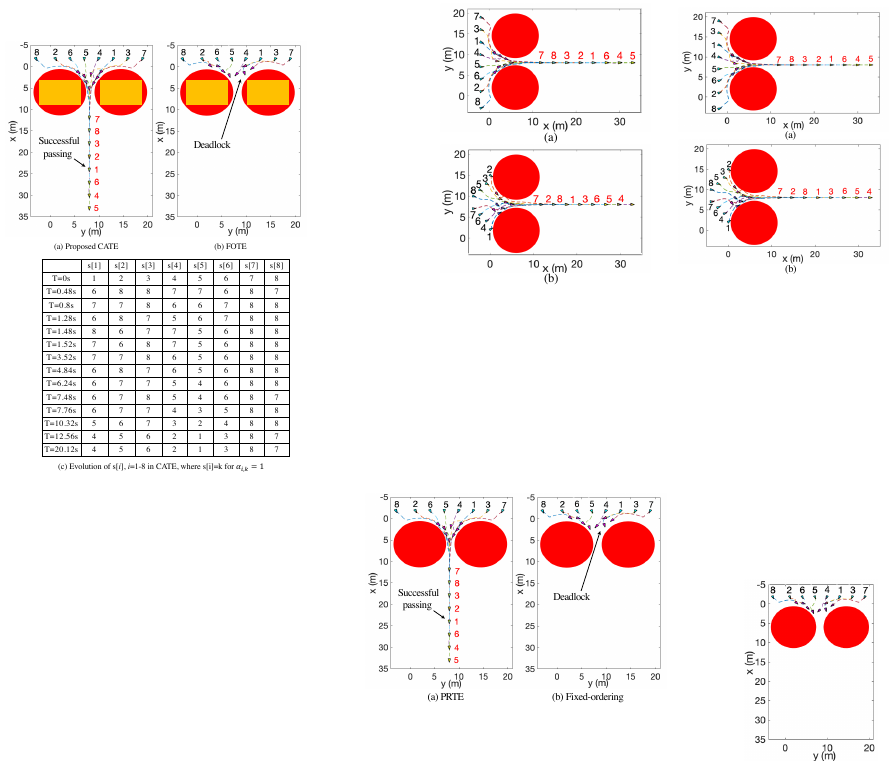}
\caption{\textbf{Fifth simulation: Feasibility of the CATE algorithm when passing through a narrow gap}. Trajectories comparison between the CATE algorithm~\eqref{resilient_prioritization_optimization} (see the successful passing in subfigure (a)) and the fixed-ordering methods (see the deadlock and failure in subfigure (b)), respectively. (c) The evolution of the reallocated ordering sequences $s[i], i\in\mathbb{Z}_1^8,$ for robot $i, i\in\mathcal V,$ satisfying $s[i]=k, \alpha_{i,k}=1$ in Fig.~\ref{in_out} (a). 
(The blue, magenta, and yellow triangles denote the initial, gap-entrance, and final positions of robots, respectively, and the red disks are the obstacles.)
}
\label{in_out}
\end{figure}

As shown on the left side of Table~\ref{table_EM} and Fig.~\ref{success_rate}, the success rates of the proposed CATE algorithm~\eqref{resilient_prioritization_optimization} are $100\%$ all along, which verifies the feasibility advantages of the algorithm in the obstacle environments. The right side of Table~\ref{table_EM} and~Fig.~\ref{convergence_time} illustrates the mean (SD) and boxplots of the convergence time of the four approaches, where the convergence time of the proposed CATE algorithm~\eqref{resilient_prioritization_optimization} is the smallest among the ones in three baseline benchmarks. To quantitatively evaluate the multi-robot navigation efficiency, we further employ two metrics for path crossings and trajectory lengths in Table~\ref{table_EM2}. It is observed on the left side of Table~\ref{table_EM2} that if the number of obstacles and robots increases, the proposed CATE algorithm~\eqref{resilient_prioritization_optimization} not only guarantees fewer path crossings than the three baseline methods but also maintains $100\%$ success rates all along. As shown on the right side of Table~\ref{table_EM2}, the proposed CATE algorithm~\eqref{resilient_prioritization_optimization} only requires the smallest trajectory length than the ones in three baseline methods, which thus indicates the highest efficiency of the CATE algorithm as well.
Moreover, Figs.~\ref{path_crossing} and \ref{trajectory_length} present the detailed boxplots of the path-crossing and trajectory-length metrics across four approaches, demonstrating that the CATE algorithm~\eqref{resilient_prioritization_optimization} achieves the highest performance.

Since the initial positions of robots and obstacles are randomly selected in each trial for each algorithm, one has that these positions are different at each trial. Meanwhile, as shown in Table~\ref{table_EM}, there are considerable differences in the success rates between the CAPT+ORCA and the CATE algorithms \eqref{resilient_prioritization_optimization}, whereas we only count the path crossings in Table~\ref{table_EM2} when the {\it MPCM navigation} is successfully achieved. Therefore, although the CAPT+ORCA algorithm achieves fewer path crossings than the proposed CATE algorithm \eqref{resilient_prioritization_optimization} in Table~\ref{table_EM2}, it is not suitable to compare such path crossings directly. Here, we provide a detailed explanation below.
 \begin{figure}[!htb]
\centering
\includegraphics[width=7.5cm]{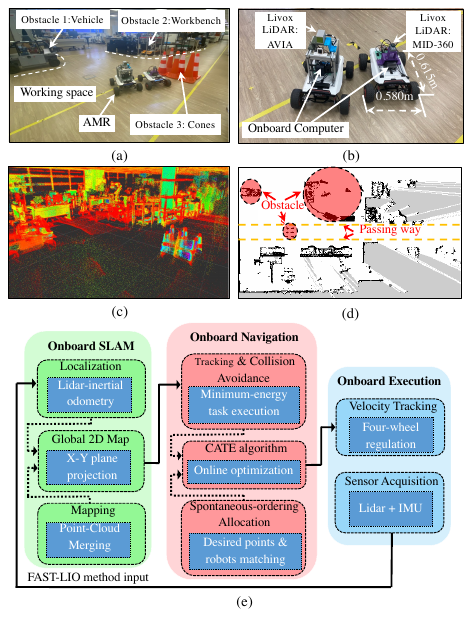}
\caption{ (a) Multi-AMR working space consisting of three different obstacles in the Robotic Research Center (RRC). (b) Two AMRs SCOUTMINI in the experiments and components. (c) 3D point map of the working space in (a) built by a popular LiDAR SLAM framework: Fast-LIO. (d) 2D octomap from the x-y plane projection of the original 3D octomap \cite{hornung2013octomap}. (e) Structure of the multi-AMR onboard navigation system, which consists of onboard SLAM, onboard navigation, and onboard execution, respectively.
}
\label{3D_Map}
\end{figure}

(i) When the number of robots and obstacles is relatively small (i.e., $N=5, 7$ and $M=4, 5$), the random selection of their positions results in a high probability that obstacles will not be positioned between the robots and the desired points. Then, such random placement causes the environment to degenerate into an almost obstacle-free one, which is more suitable for the CAPT+ORCA algorithm (this also explains why the success rate of the CAPT+ORCA algorithms is low since it is effective in simplistic environments). (ii) When the numbers of robots and obstacles become relatively large (i.e., $N=9, 11$ and $M=6, 7$), although the CAPT+ORCA algorithm achieves relatively few path crossings, its success rate is quite low (only 40\%-70\%, as indicated in Table~\ref{table_EM}). In contrast, the proposed CATE algorithm~\eqref{resilient_prioritization_optimization} not only guarantees a 100\% success rate (see Table~\ref{table_EM}) but also maintains relatively low path crossings. However, there is a tradeoff between the success rate and the number of path crossings. To achieve a 100\% success rate, the CATE algorithm~\ref{resilient_prioritization_optimization} needs to successfully deal with all complicated environments where obstacles effectively obstruct the motion from the initial positions to the desired positions of robots. 
This results in slightly more path crossings than the CAPT+ORCA algorithm. It's worth mentioning that the CATE algorithm~\eqref{resilient_prioritization_optimization} significantly outperforms the other three algorithms in terms of convergence time and trajectory length in Tables~\ref{table_EM} and \ref{table_EM2}. Therefore, the superiority of the CATE algorithm \eqref{resilient_prioritization_optimization} is verified as well.

\begin{figure*}[!htb]
\centering
\includegraphics[width=15.0cm]{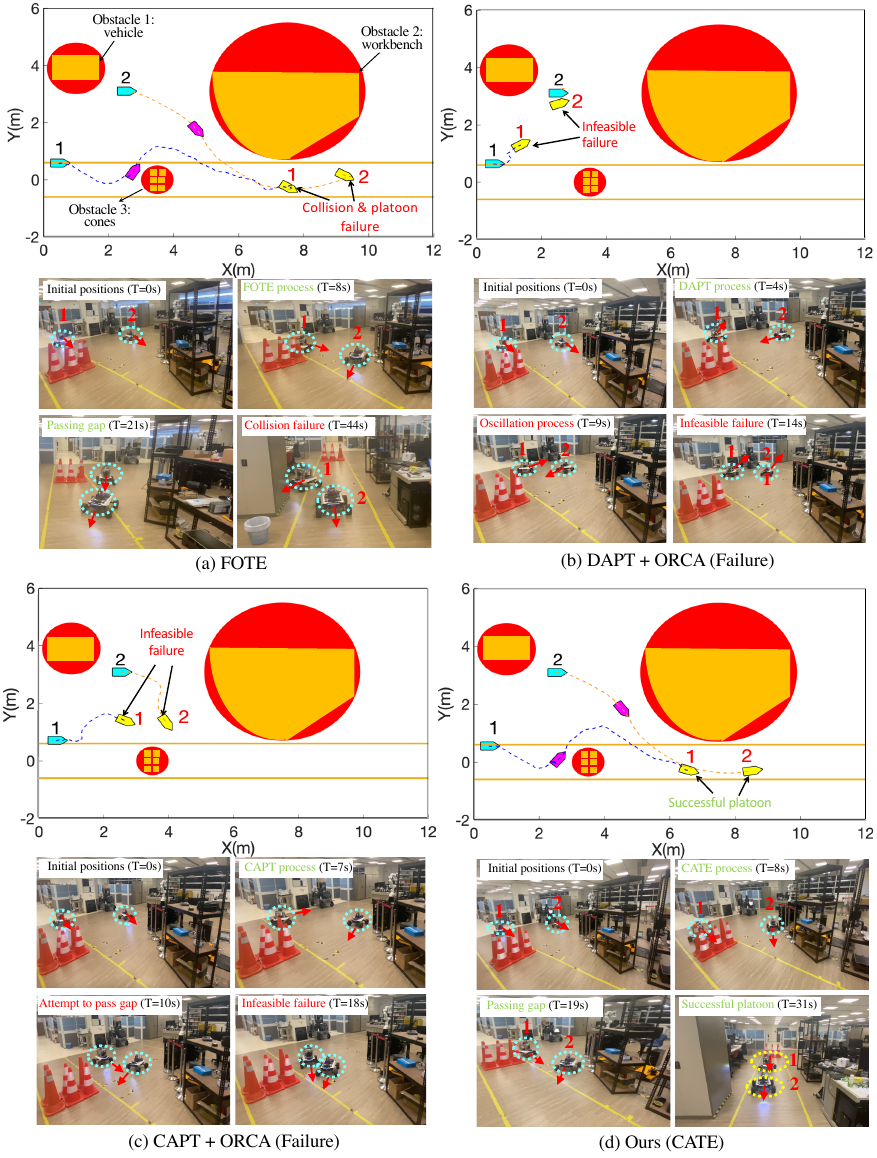}
\caption{\textbf{First experiment: Comparing platoon formations forming and snapshots using different methods in Robotic Research Center, NTU.} (a) FOTE. (b) DAPT + ORCA. (c) CAPT + ORCA. (d) Proposed CATE. The abbreviations, blue, magenta, and yellow triangles, red circles, dashed colored lines have the same meanings as those in Figs.~\ref{sota_benchmark} and \ref{in_out}.
}
\label{experiments_benchmark}
\end{figure*}

Additionally, Fig.~\ref{example_3} illustrates snapshots of nine examples of eleven robots governed by the CATE algorithm~\eqref{resilient_prioritization_optimization} from random initial positions (yellow triangles) to the ``arrow" formation (blue triangles)  in seven-obstacle environments with random size and initial positions. It is observed that the ``arrow formation" is achieved by robots with nine different ordering sequences, which verifies the flexibility to different environments. It is still worth mentioning that three failure examples using three different baseline methods are illustrated in Fig.~\ref{failure}, where the detailed reasons of failure are elaborated in Remark~\ref{failure_reason} later. Therefore, the performance of the CATE algorithm~\eqref{resilient_prioritization_optimization} outperforms the three baselines.

In the third simulation (see Fig.~\ref{moving_obstacle}), to show the effectiveness of the CATE algorithm (\eqref{resilient_prioritization_optimization} to deal with dynamic obstacles, we conduct additional simulations of eleven-robot {\it MPCM} arrow navigation with seven dynamic obstacles. It is observed in Fig.~\ref{moving_obstacle} that eleven robots from initial positions (blue triangles) can still form an arrow formation (yellow triangles) with arbitrary ordering sequences even if the seven obstacles are moving at different velocities. Therefore, the extension to dynamic obstacles is verified as well.

%In the second simulation (see Fig.~\ref{sota_benchmark2}), to demonstrate the feasibility advantages of the proposed CATE algorithm \eqref{resilient_prioritization_optimization}  in the obstacle environments compared with other techniques, we consider a more complex situation where eleven robots form an arrow formation in the presence of seven different-size circular obstacles satisfying \textbf{A2, A3, A4}. It is observed that eleven robots in Fig.~\ref{sota_benchmark2} (b) and~(d) exhibit infeasible formation failure because the DAPT + ORCA in Fig.~\ref{sota_benchmark2} (b) and CAPT+ ORCA in Fig.~\ref{sota_benchmark2} (c) are both the optimal problems for only robots without obstacles,  which thus may fail to find feasible space with additional obstacles. Despite Fig.~\ref{sota_benchmark2} (a) governed by the FOTE method can achieve arrow formation with fixed ordering sequences, there still exist lots of path crossings that reduce the efficiency of the formation forming. Comparing Fig.~\ref{sota_benchmark2} (b) and (d) together, it is observed that the proposed CATE algorithm \eqref{resilient_prioritization_optimization} can achieve the arrow formation with flexible ordering sequences in the obstacle environments, which thus reduce path crossings to a great extent.

In the fourth simulation (see Fig.~\ref{rec_obstacle}), to show the adaptability of the CATE algorithm \eqref{resilient_prioritization_optimization} in higher-dimensional Euclidean space, we also consider eight robots forming a platoon in 3D with two ball obstacles.
The initial positions of robots and desired points are set satisfying \textbf{A2, A3, A4}.  It is observed in Fig.~\ref{rec_obstacle} (a)-(b) that the platoon is formed with flexible orderings, which implicitly reduce the path crossings.

In the fifth simulation (see Fig.~\ref{in_out}), to show the feasibility of the CATE algorithm \eqref{resilient_prioritization_optimization} for possible deadlocks, we compare our algorithm with the fixed-ordering strategy in a special case of eight robots passing through a narrow gap. In particular, we use two big disk obstacles centered at $\bold{x}_1^o=[6, 2]\t, \bold{x}_2^o=[6,14]\t$ respectively with the same radius $r_1^o=r_2^o=5.5$ to occupy most of the passing space and leave a narrow gap deliberately. The initial positions of robots and desired points are set satisfying \textbf{A2, A3, A4}. 
As shown in Fig.~\ref{in_out}~(a), even though eight robots governed by the CATE algorithm \eqref{resilient_prioritization_optimization} are congestive at the entrance of the narrow gap (magenta vehicles), they can get rid of the congestion and finally, form a platoon (yellow vehicles)
by flexibly reallocating their desired points with distinct ordering sequences, which verifies the strong feasibility of the CATE algorithm \eqref{resilient_prioritization_optimization}. However, one observes in Fig.~\ref{in_out} (b) that eight robots governed by the fixed-ordering methods fail to pass the gap and are trapped in deadlocks (magenta vehicles). Moreover, Fig.~\ref{in_out} (c) illustrates the evolution of the reallocated ordering sequences $s[i], i\in\mathbb{Z}_1^8,$ for robot $i, i\in\mathcal V,$ satisfying $s[i]=k, \alpha_{i,k}=1$,  which converges to $s[1]=4, s[2]=5, s[3]=6, s[4]=2, s[5]=1, s[6]=3, s[7]=8, s[8]=7$ such that $\lim_{t\rightarrow\infty} \bold{x}_i(t)-\bold{x}_{s[i]}^d(t)=0$. Note that the final subscripts $s[i], i\in\mathbb{Z}_1^8$ are consistent with ones in Fig.~\ref{in_out} (a).

\begin{figure}[!htb]
\centering
\includegraphics[width=6.0cm]{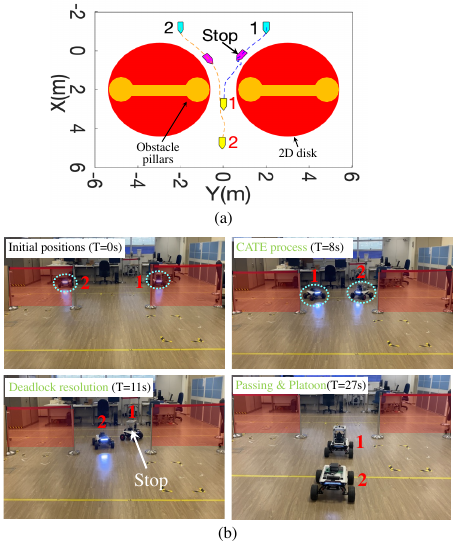}
\caption{\textbf{Second experiment: Mimic of the narrow-gap passing in Fig.~\ref{in_out}.}
(a) Trajectories of two AMRs passing through a narrow gap between two big disk obstacles, where the symbols have the same meaning as in Fig.~\ref{in_out}.
(b) Experimental snapshots where AMR $1$ stops at the entrance of the narrow gap to let AMR $2$ pass first at $ t=11$ s. The stainless steel pillars and red translucent rectangles mimic the virtual obstacles on both sides of the gap. 
}
\label{exp_inout}
\end{figure}

\subsection{Experiments with Two AMRs}
In this subsection, we employ two AMRs (SCOUTMINI\footnote{SCOUTMINI: \href{https://global.agilex.ai/products/scout-mini}{https://global.agilex.ai/products/scout-mini}}) to validate Theorem~\ref{theo_CATE_obstacle}. As shown in Fig.~\ref{3D_Map} (a), we establish a multi-AMR workspace, which consists of three different obstacles of a static vehicle, a workbench, and six cones occupying the way. As shown in Fig.~\ref{3D_Map} (b), each AMR is $0.615$m in length and $0.58$m in width and equipped with an onboard computer: NVIDIA Jetson Xavier NX3\footnote{NVIDIA Jetson Xavier NX3: \href{https://www.nvidia.com/en-us/autonomous-machines/embedded-systems/jetson-xavier-nx/}{https://www.nvidia.com/en-us/autonomous-machines/embedded-systems/jetson-xavier-nx/}} and a Livox Lidar: MID-360\footnote{Livox Lidar MID-360: \href{https://www.livoxtech.com/mid-360}{https://www.livoxtech.com/mid-360}} or AVIA\footnote{Livox Lidar AVIA: \href{https://www.livoxtech.com/avia}{https://www.livoxtech.com/avia}}. However, since the AMR is not equipped with any localization equipment (such as an on-board GPS or off-board motion capture cameras), we utilize the Lidar with the popular FAST-LIO SLAM algorithm~\cite{xu2022fast} to build a 3D point map (see Fig.~\ref{3D_Map} (c)) and localize itself. Then, we get a 2D map from the projection of the octomap using octotree \cite{hornung2013octomap} (i.e., red circles in Fig.~\ref{3D_Map} (d)).  Fig.~\ref{3D_Map} (e) illustrates the detailed multi-AMR navigation system.  
%Then, the low-level velocity tracking will be well achieved using algorithms in~\cite{hu2024ordering}. 
%In what follows,  to verify the efficacy of the CATE algorithm~\eqref{resilient_prioritization_optimization} in practice, we conduct two experiments in obstacle environments of Fig.~\ref{3D_Map} (b). 
Due to the limited workspace, the velocity limit in~\eqref{robot_dynamic} is set to be $u_{\max}=0.2$m/s, where 
the sensing and collision radii are selected to be $R=1.5$m and $r=1$m by \textbf{A5}. The weights $b, c$ in~\eqref{resilient_prioritization_optimization} are set to be $b=10000, c=100$, where the penalty parameter is set to be $\varpi=1000$. 
The design of the distributed estimators $\widehat{v}_i$ for $v_d$ in~\eqref{resilient_prioritization_optimization} is similar to Section~\ref{algo_simulations} \cite{hong2006tracking}.

In the first experiment (see Fig.~\ref{experiments_benchmark}), we set the initial positions of AMRs to be $\bold{x}_1=[0.5, 0.6]\t$m, $\bold{x}_2=[2.5,3.1]\t$m which satisfy \textbf{A3-A4}. Due to the limited working space, it follows from \textbf{A2} that the initial positions of the desired points are set to be $\bold{x}_1^{d}(0)=[5.0, 0.0]\t$m, $\bold{x}_2^d(0)=[8.0, 0.0]\t$m with a constant velocity $\bold{v}_d=[0.1, 0]\t$m/s. Fig.~\ref{experiments_benchmark} illustrates the trajectories and snapshots of the four experimental cases of the {\it MPCM} platoon using four different techniques. In particular, it is observed in Figs.~\ref{experiments_benchmark} (b) (d) that two AMRs governed by DAPT+ORCA and CAPT+ORCA from the same initial positions (blue vessls) will cause infeasible failure due to the occupation of obstacles, which is similar to the simulation results in Fig.~\ref{failure}. Since there is yet no method to achieve {\it MPCM navigation} in obstacle environments directly, the main reason for the failure is that two separate optimization methods, namely DAPT (CAPT) and ORCA may cause conflicts in their solution spaces, resulting in no solution. 
For more details, please refer to Remark~\ref{failure_reason}. As for Fig.~\ref{experiments_benchmark}~(a), although two AMRs governed by FOTE can temporarily avoid obstacles in the beginning, they collide with objects again. Comparing Figs.~\ref{experiments_benchmark} (a) and (d) together, it is observed that two AMRs governed by the proposed CATE algorithm~\eqref{resilient_prioritization_optimization} can successfully form the platoon, which thus show the superiority of  CATE algorithm~\eqref{resilient_prioritization_optimization} in the obstacle environments.

In the second experiment (see Fig.~\ref{exp_inout}), to demonstrate the feasibility of the CATE algorithm \eqref{resilient_prioritization_optimization} in practice, we mimic the fourth simulations of Fig.~\ref{in_out} by positioning two disk obstacles to establish a narrow gap for two AMRs. Precisely,  the centers and radii of obstacles are set to be $\bold{x}_1^o=[2, -3]\t$m, $\bold{x}_2^o=[2, 3]\t$m, $r_1^o=2.25$m, $r_2^o=2.25$m, respectively.
Initially, it follows from \textbf{A3, A4} that two AMRs are placed at symmetric positions, i.e., $\bold{x}_1(0)=[-1.2, 2]\t$m, $\bold{x}_2(0)=[-1.2, -2]\t$m, which deliberately make two AMRs trapped in a deadlock. Fig.~\ref{exp_inout}~(a) illustrates that two AMRs pass through the narrow gap, where AMR $1$ stops at the entrance to let AMR $2$ pass first. Then, the deadlock is eliminated and a successful platoon is formed. Moreover, to illustrate the narrow-gap passing experiments, Fig.~\ref{exp_inout} (b) depicts four experimental snapshots, where AMR $1$ stops at the gap at $t=11$s and two AMRs form a platoon with an ordering sequence $\{2,1\}$ at $t=27$s. The feasibility of the CATE algorithm~\eqref{resilient_prioritization_optimization} is validated.

\begin{remark}
\label{failure_reason}
For the failure behaviors of the previous three baseline methods in obstacle environments, the main reasons are given below. (i) For the FOTE algorithm, despite that the success rate is relatively high in Fig.~\ref{success_rate}, some failures may still occur due to deadlocks between the attraction from predefined robot-point pairs and the repulsion from robot-robot avoidance, where the illustrative failure example is given in Fig.~\ref{failure} (a).
(ii) For the DAPT+ORCA method,  according to Section~\ref{base_line}, unfortunately, there is yet no method to achieve the {\it MPCM navigation} in an obstacle environment directly, which thus motivates the combination of the two separated optimization methods together, namely, DAPT+ORCA. Here, DAPT accounts for the optimization of decoupled robot-point assignment and robot-robot avoidance to generate non-intersecting paths in obstacle-free environments, whereas ORCA accounts for the optimization of robot-obstacle avoidance. However, the simple combination of such two kinds of optimization may cause conflicts in their solution spaces, resulting in no solution, such as Figs.~\ref{failure} (b) and \eqref{experiments_benchmark} (b), respectively. 
(iii) Analogous to the DAPT+ORCA method, for the CAPT+ORCA method, where the CAPT optimization accounts for the concurrent allocation and planning to generate non-intersecting paths in obstacle-free environments, and the ORCA accounts for additional robot-obstacle avoidance. However, such two kinds of optimization may also result in infeasible solutions due to the optimization conflicts, as shown in Figs.~\ref{failure} (c) and \eqref{experiments_benchmark} (c), respectively.
\end{remark}

\section{conclusion}
\label{sec_conclusion}

\subsection{Summary}
In this paper, we have analyzed the core challenges to achieve the {\it MPCM navigation in obstacle environments} and summarized the potential infeasibility and low computational efficiency. Based on these features, we have proposed the CATE algorithm that guarantees feasible solutions for an online-constrained optimization, which can effectively reduce path crossings and trajectory length in obstacle environments. We have conducted extensive statistical simulations and real AMR experiments with comparisons against state-of-the-art baseline methods to validate its effectiveness,  adaptability, scalability, and feasibility in obstacle environments.

\subsection{Limitations and Future Work}
\label{subsec_limit_future}
Despite our effort and exploration for the {\it MPCM navigation in obstacle environments}, our work still has some limitations, which require further investigation. Particularly, (i)  The proposed CATE algorithm is essentially an MIQP solution, where the computational cost increases exponentially with the number of robots, making it unsuitable for large-scale scenarios. Future work will consider the fully distributed algorithm and include a theoretical analysis.
 (ii) This paper only considers the simplest smooth circle to cover the irregular obstacles for the usage of CBF constraints, which may sacrifice additional space. Future work will consider more efficient yet non-smooth obstacle boundaries, such as polygons. (iii) This paper only considers the simple integrator dynamics for robots, which are not practical in real applications. Future work will further investigate the complicated underactuated robot dynamics. (iv) This work only considers the scenario of static and moving obstacles at known velocities. Future work will consider the complex {\it MPCM navigation} with obstacles of random velocities.
(iv) The proposed CATE algorithm cannot be utilized to directly determine the hardness of Problem~\ref{label_pro} under the additional constraints of robot dynamics, moving obstacles, varying optimization objectives, and 3D environments. Future work will investigate the formal analysis of the problem's hardness that accounts for these additional constraints.

\section*{Acknowledgement}
We would like to thank Prof. Kiril Solovey from Israel Institute of Technology for the valuable discussions on the hardness of the problem in unlabeled multi-robot motion planning.

\bibliographystyle{IEEEtran}
\bibliography{IEEEabrv,ref}

\end{document}